%% file: paper.tex
\documentclass{article}

\usepackage[final]{nips_2018}

\usepackage[utf8]{inputenc} % allow utf-8 input
\usepackage[T1]{fontenc}    % use 8-bit T1 fonts
\usepackage{hyperref}       % hyperlinks
\usepackage{url}            % simple URL typesetting
\usepackage{booktabs}       % professional-quality tables
\usepackage{amsfonts}       % blackboard math symbols
\usepackage{nicefrac}       % compact symbols for 1/2, etc.
\usepackage{microtype}      % microtypography
\usepackage{graphicx}
\usepackage{booktabs} % for professional tables
\usepackage{amsmath,amssymb}
\usepackage{amsthm}
\usepackage{dsfont}
\usepackage{mathtools}
\usepackage{subcaption}
\usepackage[draft]{fixme}
\usepackage{bbold}
\fxsetup{inline,nomargin,theme=color}
\usepackage{array}
\newcolumntype{L}{>{\centering\arraybackslash}m{4cm}}
\usepackage{lscape}

\input{header}

\title{Why Is My Classifier Discriminatory?}

\author{
  Irene Y. Chen \\
  MIT\\
  \texttt{iychen@mit.edu} \\
  \And
  Fredrik D. Johansson \\
  MIT\\
  \texttt{fredrikj@mit.edu} \\
  \And
  David Sontag \\
  MIT \\
  \texttt{dsontag@csail.mit.edu} \\
}

\begin{document}

\maketitle

\begin{abstract}
Recent attempts to achieve fairness in predictive models focus on the balance between fairness and accuracy. In sensitive applications such as healthcare or criminal justice, this trade-off is often undesirable as any increase in prediction error could have devastating consequences. In this work, we argue that the fairness of predictions should be evaluated in context of the data, and that unfairness induced by inadequate samples sizes or unmeasured predictive variables should be addressed through data collection, rather than by constraining the model. We decompose cost-based metrics of discrimination into bias, variance, and noise, and propose actions aimed at estimating and reducing each term. Finally, we perform case-studies on prediction of income, mortality, and review ratings, confirming the value of this analysis. We find that data collection is often a means to reduce discrimination without sacrificing accuracy.
\end{abstract}

%
% INTRODUCTION
%
\section{Introduction}
As machine learning algorithms increasingly affect decision making in society, many have raised concerns about the fairness and biases of these algorithms, especially in applications to healthcare or criminal justice, where human lives are at stake~\citep{angwin2016machine,barocas2016big}. It is often hoped that the use of automatic decision support systems trained on observational data will remove human bias and improve accuracy. However, factors such as data quality and model choice may encode unintentional discrimination, resulting in systematic disparate impact.

We study fairness in prediction of outcomes such as recidivism, annual income, or patient mortality. Fairness is evaluated with respect to \emph{protected groups} of individuals defined by attributes such as gender or ethnicity~\citep{ruggieri2010data}. Following previous work, we measure discrimination in terms of differences in prediction cost across protected groups~\citep{calders2010three,dwork2012fairness, feldman2015certifying}. Correcting for issues of data provenance and historical bias in labels is outside of the scope of this work. Much research has been devoted to constraining models to satisfy cost-based fairness in prediction, as we expand on below. \emph{The impact of data collection on discrimination has received comparatively little attention}. 

Fairness in prediction has been encouraged by adjusting models through regularization~\citep{bechavod2017learning,kamishima2011fairness}, constraints~\citep{kamiran2010discrimination,zafar2017fairness}, and representation learning~\citep{zemel2013learning}. These attempts can be broadly categorized as model-based approaches to fairness. Others have applied data preprocessing to reduce discrimination~\citep{hajian2013methodology,feldman2015certifying,calmon2017optimized}. For an empirical comparison, see for example~\citet{friedler2018comparative}. Inevitably, however, restricting the model class or perturbing training data to improve fairness may harm predictive accuracy~\citep{corbett2017algorithmic}. 

A \emph{tradeoff} of predictive accuracy for fairness is sometimes difficult to motivate when predictions influence high-stakes decisions. In particular, post-hoc correction methods based on randomizing predictions~\citep{hardt2016equality,pleiss2017fairness} are unjustifiable for ethical reasons in clinical tasks such as severity scoring. Moreover, as pointed out by \citet{woodworth2017learning}, post-hoc correction may lead to suboptimal predictive accuracy compared to other equally fair classifiers. 

Disparate predictive accuracy can often be explained by insufficient or skewed sample sizes or inherent unpredictability of the outcome given the available set of variables. With this in mind, we propose that fairness of predictive models should be analyzed in terms of model bias, model variance, and outcome noise \emph{before} they are constrained to satisfy fairness criteria. This exposes and separates the adverse impact of inadequate data collection and the choice of the model on fairness. 
The cost of fairness need not always be one of predictive accuracy, but one of investment in data collection and model development. In high-stakes applications, the benefits often outweigh the costs.

In this work, we use the term ``discrimination" to refer to specific kinds of differences in the predictive power of models when applied to different protected groups. In some domains, such differences may not be considered discriminatory, and it is critical that decisions made based on this information are sensitive to this fact. For example, in prior work, researchers showed that causal inference may help uncover which sources of differences in predictive accuracy introduce unfairness \citep{kusner2017counterfactual}. In this work, we assume that observed differences are considered discriminatory and discuss various means of explaining and reducing them. 

\paragraph{Main contributions}
We give a procedure for analyzing  discrimination in predictive models with respect to cost-based definitions of group fairness, emphasizing the impact of data collection. First, we propose the use of bias-variance-noise decompositions for separating sources of discrimination. Second, we suggest procedures for estimating the value of collecting additional training samples. Finally, we propose the use of clustering for identifying subpopulations that are discriminated against to guide additional variable collection. We use these tools to analyze the fairness of common learning algorithms in three tasks: predicting income based on census data, predicting mortality of patients in critical care, and predicting book review ratings from text. We find that the accuracy in predictions of the mortality of cancer patients vary by as much as $20\%$ between protected groups. In addition, our experiments confirm that discrimination level is sensitive to the quality of the training data.

%
% BACKGROUND
%
\section{Background} 
\label{sec:background}
We study fairness in prediction of an outcome $Y \in \mathcal{Y}$. Predictions are based on a set of covariates $X \in \mathcal{X} \subseteq \mathbb{R}^k$ and a \emph{protected attribute} $A \in \mathcal{A}$. In mortality prediction, $X$ represents the medical history of a patient in critical care, $A$ the self-reported ethnicity, and $Y$ mortality. A model is considered fair if its errors are distributed similarly across protected groups, as measured by a cost function $\gamma$. Predictions learned from a training set $d$ are denoted $\hat{Y}_d := h(X,A)$ for some $h : \mathcal{X}\times\mathcal{A} \rightarrow \mathcal{Y}$ from a class $\mathcal{H}$. The protected attribute is assumed to be binary, $\cA = \{0,1\}$, but our results generalize to the non-binary case. A dataset $d = \{(x_i, a_i, y_i)\}_{i=1}^n$ consists of $n$ samples distributed according to $p(X, A, Y)$. When clear from context, we drop the subscript from $\hat{Y}_d$. %Finally, we refer to the subpopulation defined by $p(X\mid A=a)$ as \emph{protected group} $a$.

A popular cost-based definition of fairness is the \textit{equalized odds} criterion, which states that a binary classifier $\hat{Y}$ is fair if its false negative rates (FNR) and false positive rates (FPR) are equal across groups~\citep{hardt2016equality}. We define FPR and FNR with respect to protected group $a\in \cA$ by
\begin{align*}
\fpr_{a}(\hat{Y}) \coloneqq \mathbb{E}_X[\hat{Y} \mid  Y = 0, A=a], \;\;\;\; 
\fnr_{a}(\hat{Y}) \coloneqq \mathbb{E}_X[1-\hat{Y} \mid  Y = 1, A=a]~.
\end{align*}
Exact equality, $\fpr_{0}(\hat{Y}) = \fpr_{1}(\hat{Y})$, is often hard to verify or enforce in practice. Instead, we study the \emph{degree} to which such constraints are violated. More generally, we use differences in \emph{cost functions} $\gamma_a$ between protected groups $a \in \cA$ to define the \emph{level of discrimination} $\Gamma$,
\begin{equation}\label{eq:alphadisc}
\Gamma^\gamma(\hat{Y}) := \left| \gamma_0(\hat{Y}) - \gamma_1(\hat{Y}) \right| ~.
\end{equation}
In this work we study cost functions $\gamma_a \in \{\fpr_a, \fnr_a, \zo_a\}$ in binary classification tasks, with $\zo_a(\hat{Y}) := \mathbb{E}_X[\mathds{1}[\hat{Y} \neq Y] \mid A = a]$ the \emph{zero-one loss}. In regression problems, we use the group-specific \emph{mean-squared error} $\mse_a := \mathbb{E}_X[(\hat{Y}-Y)^2 \mid A = a]$. 
According to \eqref{eq:alphadisc}, predictions $\hY$ satisfy equalized odds on $d$ if $\Gamma^{\scfpr}(\hY) = 0$ \emph{and} $\Gamma^{\scfnr}(\hY) = 0$. 

\paragraph{Calibration and impossibility} 
A score-based classifier is \emph{calibrated} if the prediction score assigned to a unit equals the fraction of positive outcomes for all units assigned similar scores. It is impossible for a classifier to be calibrated in every protected group and satisfy multiple cost-based fairness criteria at once, unless accuracy is perfect or base rates of outcomes are equal across groups~\citep{chouldechova2017fair}. A relaxed version of this result~\citep{kleinberg2016inherent} applies to the discrimination level $\Gamma$. Inevitably, both constraint-based methods and our approach are faced with a choice between which fairness criteria to satisfy, and at what cost. 

%
% SOURCES OF UNFAIRNESS
%
\section{Sources of perceived discrimination}
\label{sec:sources}
There are many potential sources of  discrimination in predictive models. In particular, the choice of hypothesis class $\mathcal{H}$ and learning objective has received a lot of attention~\citep{calders2010three,zemel2013learning,fish2016confidence}. However, data collection---the chosen set of predictive variables $X$, the sampling distribution $p(X, A, Y)$, and the training set size $n$---is an equally integral part of deploying fair machine learning systems in practice, and it should be guided to promote fairness. 
Below, we tease apart sources of discrimination through bias-variance-noise decompositions of cost-based fairness criteria.
In general, we may think of noise in the outcome as the effect of a set of unobserved variables $U$, potentially interacting with $X$. Even the optimal achievable error for predictions based on $X$ may be reduced further by observing parts of $U$. In Figure~\ref{fig:scenarios}, we illustrate three common learning scenarios and study their fairness properties through bias, variance, and noise. 

%
% BIG FIGURE
%
\begin{figure*}[tbp!]
\centering
\begin{subfigure}[b]{0.32\textwidth}
\centering
\includegraphics[width=\textwidth]{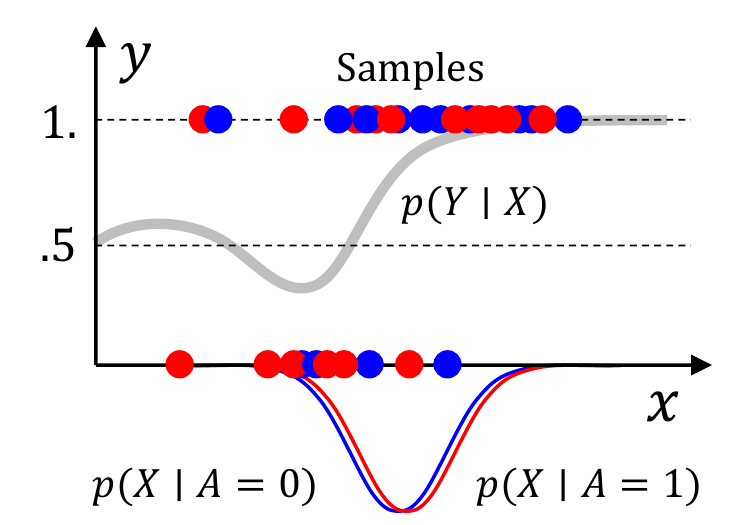}
\caption{\label{fig:scen_groups}For identically distributed protected groups and unaware outcome (see below), bias and noise are equal in expectation. Perceived discrimination is only due to variance.}
\end{subfigure}
\;
\begin{subfigure}[b]{0.32\textwidth}
\centering
\includegraphics[width=\textwidth]{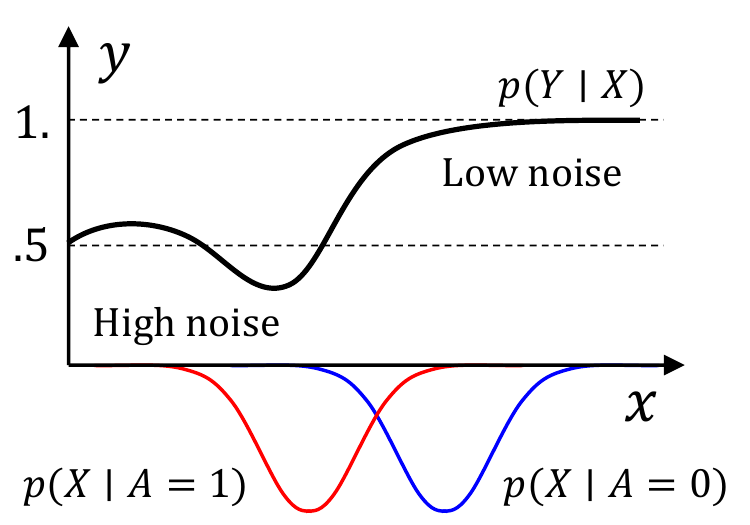}
\caption{\label{fig:scen_entropy}Heteroskedastic noise, i.e. $\exists x, x': N(x) \neq N(x')$, may contribute to discrimination even for an optimal model if protected groups are not identically distributed. }
\end{subfigure}
\;
\begin{subfigure}[b]{0.32\textwidth}
\centering
\includegraphics[width=\textwidth]{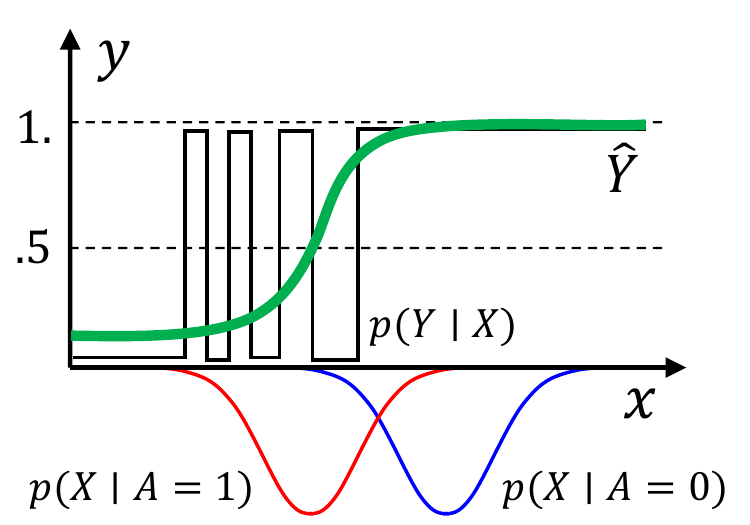}
\caption{\label{fig:scen_capacity}One choice of model may be more suited for one protected group, even under negligible noise and variance, resulting in a difference in expected bias, $\olB_0 \neq \olB_1$.}
\end{subfigure}
\caption{\label{fig:scenarios}Scenarios illustrating how properties of the training set and model choice affect perceived discrimination in a binary classification task, under the assumption that outcomes and predictions are \emph{unaware}, i.e. $p(Y\mid X, A) = p(Y\mid X)$ and $p(\hat{Y}\mid X, A) = p(\hat{Y}\mid X)$. Through bias-variance-noise decompositions (see Section~\ref{sec:bvd}), we can identify which of these dominate in their effect on fairness. We propose procedures for addressing each component in Section~\ref{sec:solutions}, and use them in experiments (see Section~\ref{sec:experiments}) to mitigate discrimination in income prediction and prediction of ICU mortality.}
\end{figure*}

To account for randomness in the sampling of training sets, we redefine discrimination level \eqref{eq:alphadisc} in terms of the \emph{expected} cost $\olgamma_{a}(\hat{Y}) := \mathbb{E}_D[ \gamma_{a}(\hat{Y}_D)]$ over draws of a random training set $D$.
\begin{thmdef}\label{def:exp_disc}
The \emph{expected discrimination level} $\olGamma(\hat{Y})$ of a predictive model $\hY$ learned from a random training set $D$, is
\begin{align}
\olGamma(\hat{Y}) & := \left| \E_D\left[\gamma_0(\hY_D) - \gamma_1(\hY_D) \right] \right| 
 \;= \left| \olgamma_0(\hat{Y}) - \olgamma_1(\hat{Y}) \right| ~.  \nonumber
\end{align}
\end{thmdef}
$\olGamma(\hY)$ is not observed in practice when only a single training set $d$ is available. If $n$ is small, it is recommended to estimate $\olGamma$ through re-sampling methods such as bootstrapping~\citep{efron1992bootstrap}.

%
% BIAS VARIANCE DECOMPOSITION
%
\subsection{Bias-variance-noise decompositions of discrimination level}\label{sec:bvd}
An algorithm that learns models $\hY_D$ from datasets $D$ is given, and the covariates $X$ and size of the training data $n$ are fixed. We assume that $\hat{Y}_D$ is a deterministic function $\hat{y}_D(x,a)$ given the training set $D$, e.g. a thresholded scoring function. Following \citet{domingos2000unified}, we base our analysis on decompositions of loss functions $L$ evaluated at points $(x, a)$. For decompositions of costs $\gamma_a \in \{\zo{}, \fpr{}, \fnr{}\}$ we let this be the zero-one loss, $L(y,y') = \mathds{1}[y\neq y']$ , and for $\gamma_a = \mse$, the squared loss, $L(y,y') = (y-y')^2$. We define the \emph{main prediction} $\ymain(x, a) = \argmin_{y'} \mathbb{E}_D[L(\hat{Y}_D, y')\mid X=x, A=a]$ as the average prediction over draws of training sets for the squared loss, and the majority vote for the zero-one loss. The \emph{(Bayes) optimal prediction} $y^*(x, a) = \argmin_{y'} \mathbb{E}_Y[L(Y, y')\mid X=x, A=a]$ achieves the smallest expected error with respect to the random outcome $Y$. 

\begin{thmdef}[Bias, variance and noise] Following \citet{domingos2000unified}, we define bias $B$, variance $V$ and noise $N$ at a point $(x, a)$ below. 
\begin{equation}
\arraycolsep=1.4pt\def\arraystretch{1.4}
\begin{array}{rclrcl}
B(\hat{Y}, x, a) & = & L(y^*(x, a), \ymain(x, a)) 
& N(x, a) & = & \mathbb{E}_Y[L(y^*(x, a), Y)\mid X=x, A=a] \\
V(\hat{Y}, x, a) & = & \mathbb{E}_D[L(\ymain(x, a), \hat{y}_D(x, a))] ~.
\end{array}
\label{eq:bvd_def}
\end{equation}
Here, $y^*, \hat{y}$ and $\ymain$, are all deterministic functions of $(x,a)$, while $Y$ is a random variable.
\end{thmdef}

In words, the bias $B$ is the loss incurred by the main prediction relative to the optimal prediction. The variance $V$ is the average loss incurred by the predictions learned from different datasets relative to the main prediction. The noise $N$ is the remaining loss independent of the learning algorithm, often known as the Bayes error. We use these definitions to decompose $\olGamma$ under various definitions of $\gamma_a$.

\begin{thmthm}\label{thm:bvd}
With $\olgamma_a$ the group-specific zero-one loss or class-conditional versions (e.g. FNR, FPR), or the mean squared error, $\olgamma_a$ and the discrimination level $\olGamma$ admit decompositions of the form
$$
\olgamma_a(\hat{Y}) = \underbrace{\overline{N}_a}_{\mbox{Noise}} + \underbrace{\overline{B}_a(\hat{Y})}_{\mbox{Bias}} + \underbrace{\overline{V}_a(\hat{Y})}_{\mbox{Variance}}
\;\;\; \mbox{ and } \;\;\;\;
\olGamma = \left|(\overline{N}_0 - \overline{N}_{1}) + (\overline{B}_0 - \overline{B}_{1}) + (\overline{V}_0 - \overline{V}_{1}) \right|
$$
where we leave out $\hat{Y}$ in the decomposition of $\olGamma$ for brevity. With $B, V$ defined as in \eqref{eq:bvd_def}, we have
$$
\overline{B}_a(\hat{Y}) = \mathbb{E}_{X}[B(\ymain{}, X, a) \mid A=a] \;\;\;\mbox{ and }\;\;\;
\overline{V}_a(\hat{Y}) = \mathbb{E}_{X,D}[c_v(X) V(\hat{Y}_D, X, a) \mid A=a]~.
$$ 
For the zero-one loss, $c_v(x,a)=1$ if $\hat{y}_m(x,a) = y^*(x,a)$, otherwise $c_v(x,a) = -1$. For the squared loss $c_v(x,a) = 1$. 
The noise term for population losses is $$\overline{N}_a := \E_{X}[c_n(X,a) L(y^*(X,a), Y) \mid A=a]$$ and for class-conditional losses w.r.t class $y\in\{0, 1\}$, $$\overline{N}_a(y) := \E_{X}[c_n(X,a) L(y^*(X,a), y) \mid A=a, Y=y]~.$$ 
For the zero-one loss, and class-conditional variants, $c_n(x,a) = 2\E_{D}[\mathds{1}[\hat{y}_D(x,a) = y^*(x,a)]]-1$ and for the squared loss, $c_n(x,a) = 1$. 
\end{thmthm}
\begin{thmproofsketch}
Conditioning and exchanging order of expectation, the cases of mean squared error and zero-one losses follow from \citet{domingos2000unified}. Class-conditional losses follow from a case-by-case analysis of possible errors. See the supplementary material for a full proof. \qed
\end{thmproofsketch}

Theorem~\ref{thm:bvd} points to distinct sources of perceived discrimination. Significant differences in bias $\olB_0-\olB_1$ indicate that the chosen model class is not flexible enough to fit both protected groups well (see Figure~\ref{fig:scen_capacity}). This is typical of (misspecified) linear models which approximate non-linear functions well only in small regions of the input space. 
%If more flexible models are available, these should be considered before resorting to post-hoc correction or preemptive constraining of the learned hypothesis. 
Regularization or post-hoc correction of models effectively increase the bias of one of the groups, and should be considered only if there is reason to believe that the original bias is already minimal. 

Differences in variance, $\olV_0 - \olV_1$, could be caused by differences in sample sizes $n_0, n_1$ or group-conditional feature variance $\mbox{Var}(X\mid A)$, combined with a high capacity model. Targeted collection of training samples may help resolve this issue. Our decomposition does not apply to post-hoc randomization methods~\citep{hardt2016equality} but we may treat these in the same way as we do random training sets and interpret them as increasing the variance $\olV_a$ of one group to improve fairness.
%match a high cost $\olgamma_{a'}$ in the other.

%
%\subsection{Noise and the choice of predictive variables}
%The often-made assumption that noise is negligible~\citep{domingos2000unified} is unsatisfactory in our setting---it is precisely 
When noise is significantly different between protected groups, discrimination is partially unrelated to model choice and training set size and may only be reduced by measuring additional variables. 
%The difference in expected noise $\overline{N}_0-\overline{N}_1$ represents the discrimination of a Bayes optimal classifier $y^*$.
\begin{thmprop}
If $\overline{N}_0 \neq \overline{N}_1$, no model can be 0-discriminatory in expectation without access to additional information or increasing bias or variance w.r.t. to the Bayes optimal classifier. 
\label{prop:bayes}
\end{thmprop}
\begin{proof}By definition, $\olGamma = 0 \implies (\olN_1 - \olN_0) = (\olB_0 - \olB_1) + (\olV_0 - \olV_1)$. As the Bayes optimal classifier has neither bias nor variance, the result follows immediately.
\end{proof}
In line with Proposition~\ref{prop:bayes}, 
most methods for ensuring algorithmic fairness reduce  discrimination by trading off a difference in noise for one in bias or variance. However, this trade-off is only motivated if the considered predictive model is close to Bayes optimal \emph{and} no additional predictive variables may be measured. Moreover, if noise is homoskedastic in regression settings, post-hoc randomization is ill-advised, as the difference in Bayes error $\olN_0 - \olN_1$ is zero, and discrimination is caused only by model bias or variance (see the supplementary material for a proof).

\paragraph{Estimating bias, variance and noise}
\label{sec:estimation}
Group-specific variance $\olV_a$ may be estimated through sample splitting or bootstrapping~\citep{efron1992bootstrap}. In contrast, the noise $\olN_a$ and bias $\olB_a$ are difficult to estimate when $X$ is high-dimensional or continuous. In fact, no convergence results of noise estimates may be obtained without further assumptions on the data distribution~\citep{antos1999lower}. Under some such assumptions, noise may be approximately estimated using distance-based methods~\citep{devijverkittler1982}, nearest-neighbor methods~\citep{fukunaga1987bayes,cover1967nearest}, or classifier ensembles~\citep{tumer1996estimating}. When comparing the discrimination level of two different models, noise terms cancel, as they are independent of the model. As a result, \emph{differences} in bias may be estimated even when the noise is not known (see the supplementary material).

\paragraph{Testing for significant discrimination} When sample sizes are small, perceived discrimination may not be statistically significant. In the supplementary material, we give statistical tests both for the discrimination level $\Gamma(\hat{Y})$ and the difference in discrimination level between two models $\hat{Y}, \hat{Y}'$.

%
% REDUCING DISCRIMINATION
%
\section{Reducing discrimination through data collection} \label{sec:solutions}
In light of the decomposition of Theorem~\ref{thm:bvd}, we explore avenues for reducing group differences in bias, variance, and noise without sacrificing predictive accuracy. In practice, predictive accuracy is often artificially limited when data is expensive or impractical to collect. With an investment in training samples or measurement of predictive variables, both accuracy and fairness may be improved.

%
% Training set size
%
\subsection{Increasing training set size}
\label{sec:trainsetsize}
Standard regularization used to avoid overfitting is not guaranteed to improve or preserve fairness. An alternative route is to collect more training samples and reduce the impact of the bias-variance trade-off. When supplementary data is collected from the same distribution as the existing set, covariate shift may be avoided~\citep{quionero2009dataset}. This is often achievable; labeled data may be expensive, such as when paying experts to label observations, but given the means to acquire additional labels, they would be drawn from the original distribution.
To estimate the value of increasing sample size, we predict the discrimination level $\olGamma(\hat{Y}_D)$ as $D$ increases in size. 

The curve measuring generalization performance of predictive models as a function of training set size $n$ is called a Type II \emph{learning curve}~\citep{domhan2015speeding}. We call $\olgamma_a(\hat{Y}, n) := \E[\gamma_a(\hat{Y}_{D_n})]$, as a function of $n$, the learning curve with respect to protected group $a$. We define the discrimination learning curve $\olGamma(\hat{Y}, n) \coloneqq |\olgamma_0(\hat{Y}, n) - \olgamma_1(\hat{Y}, n)|$ (see Figure~\ref{fig:adult} for an example).
Empirically, learning curves behave asymptotically as \emph{inverse power-law} curves for diverse algorithms such as deep neural networks, support vector machines, and nearest-neighbor classifiers, even when model capacity is allowed to grow with $n$~\citep{hestness2017deep,mukherjee2003estimating}. This observation is also supported by theoretical results~\citep{amari1993universal}. 

\begin{thmasmp}[Learning curves]\label{asmp:pow3}
The population prediction loss $\olgamma(\hat{Y}, n)$, and group-specific losses $\olgamma_0(\hat{Y}, n), \olgamma_1(\hat{Y}, n)$, for a fixed learning algorithm $\hat{Y}$, behave asymptotically as inverse power-law curves with parameters $(\alpha, \beta, \delta)$. That is, $\exists M, M_0, M_1$ such that for $n \geq M, n_a \geq M_a$, 
\begin{equation}
\olgamma(\hat{Y}, n) = \alpha n^{-\beta} + \delta \;\;\; \mbox{ and } \;\;\; 
\forall a\in \cA: \olgamma_a(\hat{Y}, n_a) = \alpha_a n_a^{-\beta_a} + \delta_a 
\label{eq:pow3}
\end{equation}
\end{thmasmp}
Intercepts, $\delta, \delta_a$ in \eqref{eq:pow3} represent the asymptotic bias $\olB(\hY_{D_\infty})$ and the Bayes error $\olN$, with the former vanishing for consistent estimators. Accurately estimating $\delta$ from finite samples is often challenging as the first term tends to dominate the learning curve for practical sample sizes.

In experiments, we find that the inverse power-laws model fit group conditional ($\gamma_a$) and class-conditional (\fpr, \fnr) errors well, and use these to extrapolate $\olGamma(\hY,n)$ based on estimates from subsampled data. 

%
% Variable set
%
\subsection{Measuring additional variables}
\label{sec:add_var}
When discrimination $\olGamma$ is dominated by a difference in noise, $\olN_0 - \olN_1$, fairness may not be improved through model selection alone without sacrificing accuracy (see Proposition~\ref{prop:bayes}). Such a scenario is likely when available covariates are not equally predictive of the outcome in both groups. We propose identification of clusters of individuals in which discrimination is high as a means to guide further variable collection---if the variance in outcomes within a cluster is not explained by the available feature set, additional variables may be used to further distinguish its members. 

Let a random variable $C$ represent a (possibly stochastic) clustering such that $C = c$ indicates membership in cluster $c$.
Then let $\rho_a(c)$ denote the expected prediction cost for units in cluster $c$ with protected attribute $a$. As an example, for the zero-one loss we let
$$
\rho_a^{\text{ZO}}(c) := \mathbb{E}_X[\mathds{1}[\hY \neq Y] \mid A=a, C = c],
$$
and define $\rho$ analogously for false positives or false negatives. Clusters $c$ for which $|\rho_0(c) - \rho_1(c) |$ is large identify groups of individuals for which discrimination is worse than average, and can guide targeted collection of additional variables or samples. In our experiments on income prediction, we consider particularly simple clusterings of data defined by subjects with measurements above or below the average value of a single feature $x(c)$ with $c \in \{1, \ldots, k\}$. 
In mortality prediction, we cluster patients using topic modeling.
As measuring additional variables is expensive, the utility of a candidate set should be estimated before collecting a large sample~\citep{koepke2012fast}.

%
% EXPERIMENTS
%
\section{Experiments}
\label{sec:experiments}
We analyze the fairness properties of standard machine learning algorithms in three tasks: prediction of  income based on national census data, prediction of patient mortality based on clinical notes, and prediction of book review ratings based on review text.\footnote{A synthetic experiment validating group-specific learning curves is left to the supplementary material.} We disentangle sources of discrimination by assessing the level of discrimination for the full data,estimating the value of increasing training set size by fitting Type II learning curves, and using clustering to identify subgroups where discrimination is high. In addition, we estimate the Bayes error through non-parametric techniques. 

In our experiments, we omit the sensitive attribute $A$ from our classifiers to allow for closer comparison to previous works, e.g. \citet{hardt2016equality, zafar2017fairness}. In preliminary results, we found that fitting separate classifiers for each group increased the error rates of both groups due to the resulting smaller sample size, as classifiers could not learn from other groups. As our model objective is to maximize accuracy over all data points, our analysis uses a single classifier trained on the entire population.

\subsection{Income prediction}
\label{sec:exp_income}

% ---------
% ADULT
% ---------

\begin{figure*}[tbp!]
\begin{subfigure}[b]{.44\textwidth}
\centering
\includegraphics[width=.85\columnwidth]{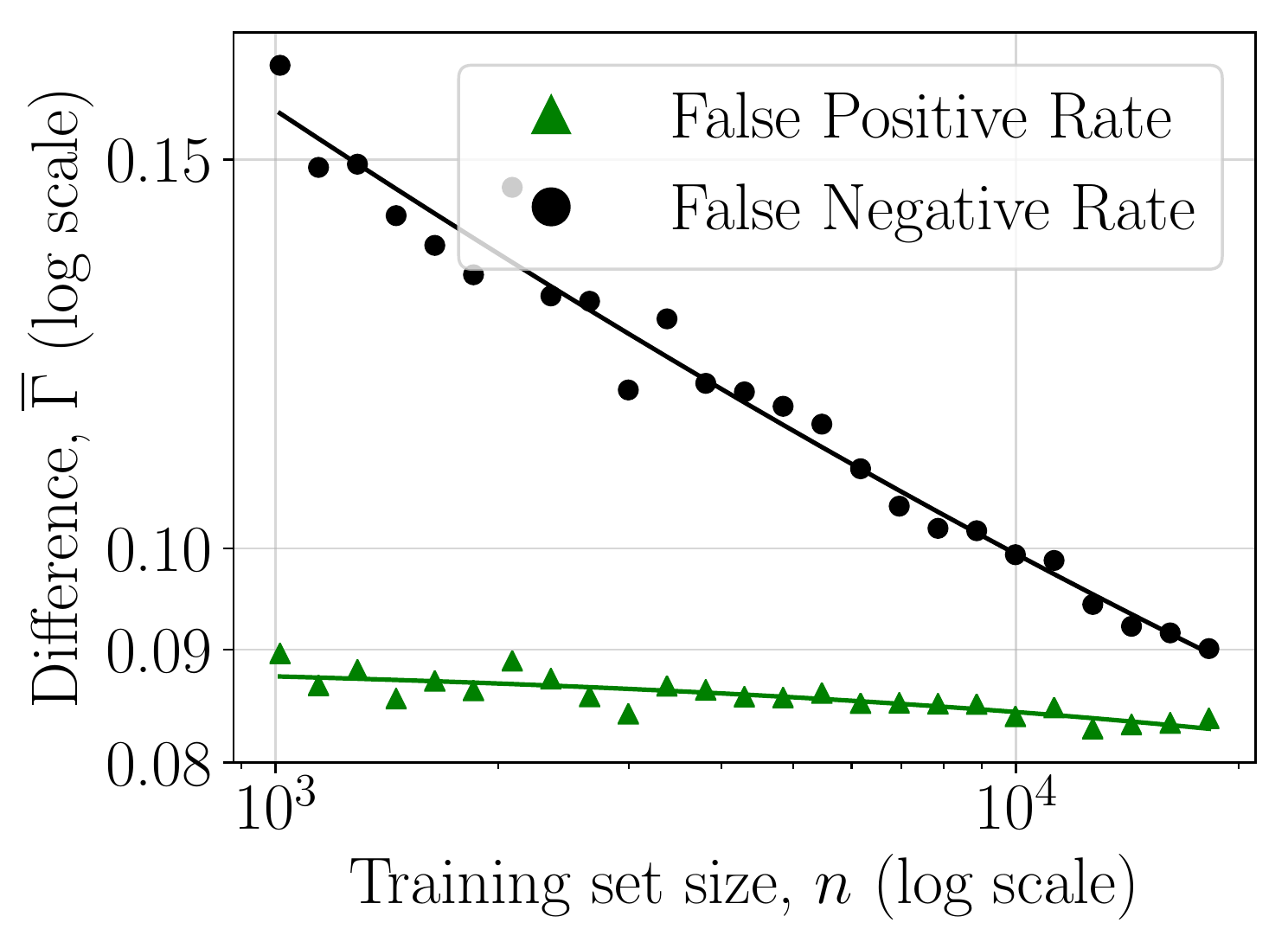}
\caption{\label{fig:adult} Group differences in false positive rates and false negative rates for a random forest classifier decrease with increasing training set size. }

\end{subfigure}
\;
\begin{subfigure}[b]{.54\textwidth}
\centering
\begin{tabular}{c|c|c|c}
Method & $E_{low}$ & $E_{up}$ & group \\
\hline
\hline
Mahalanobis &  -- & 0.29 & men \\ 
\citep{mahalanobis1936generalized} &  -- & 0.13 & women \\ 
 \hline 
Bhattacharyya & 0.001 & 0.040 & men \\
\citep{bhattacharyya1943measure} & 0.001 & 0.027 & women \\
 \hline
Nearest Neighbors & 0.10 & 0.19 & men \\
 \citep{cover1967nearest} & 0.04 & 0.07 & women \\
\end{tabular}
\caption{\label{fig:adult_error} Estimation of Bayes error lower and upper  bounds ($E_{low}$ and $E_{up}$) for zero-one loss of men and women. Intervals for men and women are non-overlapping for Nearest Neighbors.} 
\end{subfigure}
\caption{\label{fig:exp_income} Discrimination level and noise estimation in income prediction with the Adult dataset.}
\end{figure*}

Predictions of a person's salary may be used to help determine an individual's market worth, but systematic underestimation of the salary of protected groups could harm their competitiveness on the job market. The Adult dataset in the UCI Machine Learning Repository~\citep{lichman2013uci} contains 32,561 observations of yearly income (represented as a binary outcome: over or under \$50,000) and twelve categorical or continuous features including education, age, and marital status. Categorical attributes are dichotomized, resulting in a total of 105 features. 

We follow \citet{pleiss2017fairness} and strive to ensure fairness across genders, which is excluded as a feature from the predictive models. Using an 80/20 train-test split, we learn a random forest predictor, which is is well-calibrated for both groups (\citet{brier1950verification} scores of 0.13 and 0.06 for men and women). We find the difference in zero-one loss $\Gamma^{\sczo}(\hY)$ has a 95\%-confidence interval\footnote{Details for computing statistically significant discrimination can be found in the supplementary material.} $.085 \pm .069$ with decision thresholds at 0.5. At this threshold, the false negative rates are $0.388 \pm 0.026$ and $0.448 \pm 0.064$ for men and women respectively, and the false positive rates $0.111 \pm 0.011$ and $0.033 \pm 0.008$. We focus on random forest classifiers, although we found similar results for logistic regression and decision trees.

We examine the effect of varying training set size $n$ on discrimination. We fit inverse power-law curves to estimates of $\fpr(\hY,n)$ and $\fnr(\hY,n)$ using repeated sample splitting where at least 20\% of the full data is held out for evaluating generalization error at every value of $n$. We tune hyperparameters for each training set size for decision tree classifiers and logistic regression but tuned over the entire dataset for random forest. We include full training details in the supplementary material. Metrics are averaged over 50 trials. See Figure~\ref{fig:adult} for the results for random forests. Both \fpr{} and \fnr{} decrease with additional training samples. The discrimination level $\Gamma^{\scfnr}$ for false negatives decreases by a striking 40\% when increasing the training set size from 1000 to 10,000. This suggests that trading off accuracy for fairness at small sample sizes may be ill-advised. Based on fitted power-law curves, we estimate that for unlimited training data drawn from the same distribution, we would have $\Gamma^{\scfnr}(\hY) \approx 0.04$ and $\Gamma^{\scfpr}(\hY) \approx 0.08$. 

In Figure~\ref{fig:adult_error}, we compare estimated upper and lower bounds on noise ($E_{low}$ and $E_{up}$) for men and women using the Mahalanobis and Bhattacharyya distances~\citep{devijverkittler1982}, and a $k$-nearest neighbor method~\citep{cover1967nearest} with $k=5$ and 5-fold cross validation. Men have consistently higher noise estimates than women, which is consistent with the differences in zero-one loss found using all models. For nearest neighbors estimates, intervals for men and women are non-overlapping, which suggests that noise may contribute substantially to discrimination.

To guide attempts at reducing discrimination further, we identify clusters of individuals for whom false negative predictions are made at different rates between protected groups, with the method described in Section~\ref{sec:add_var}. We find that for individuals in executive or managerial occupations (12\% of the sample), false negatives are more than twice as frequent for women (0.412) as for men (0.157). For individuals in all other occupations, the difference is significantly smaller, 0.543 for women and 0.461 for men, despite the fact that the disparity in outcome base rates in this cluster is large (0.26 for men versus 0.09 for women). A possible reason is that in managerial occupations the available variable set explains a larger portion of the variance in salary for men than for women. If so, further sub-categorization of managerial occupations could help reduce  discrimination in prediction.

%

% ---------
% MIMIC
% ---------

\subsection{Intensive care unit mortality prediction}

\begin{figure*}[tbp!]
\includegraphics[width=0.9\textwidth]{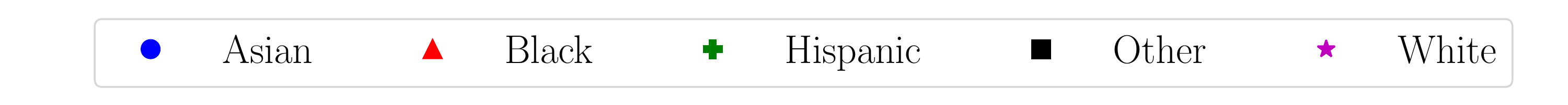}
\centering 
\begin{subfigure}[b]{.31\textwidth}
\centering
\includegraphics[width=.95\columnwidth]{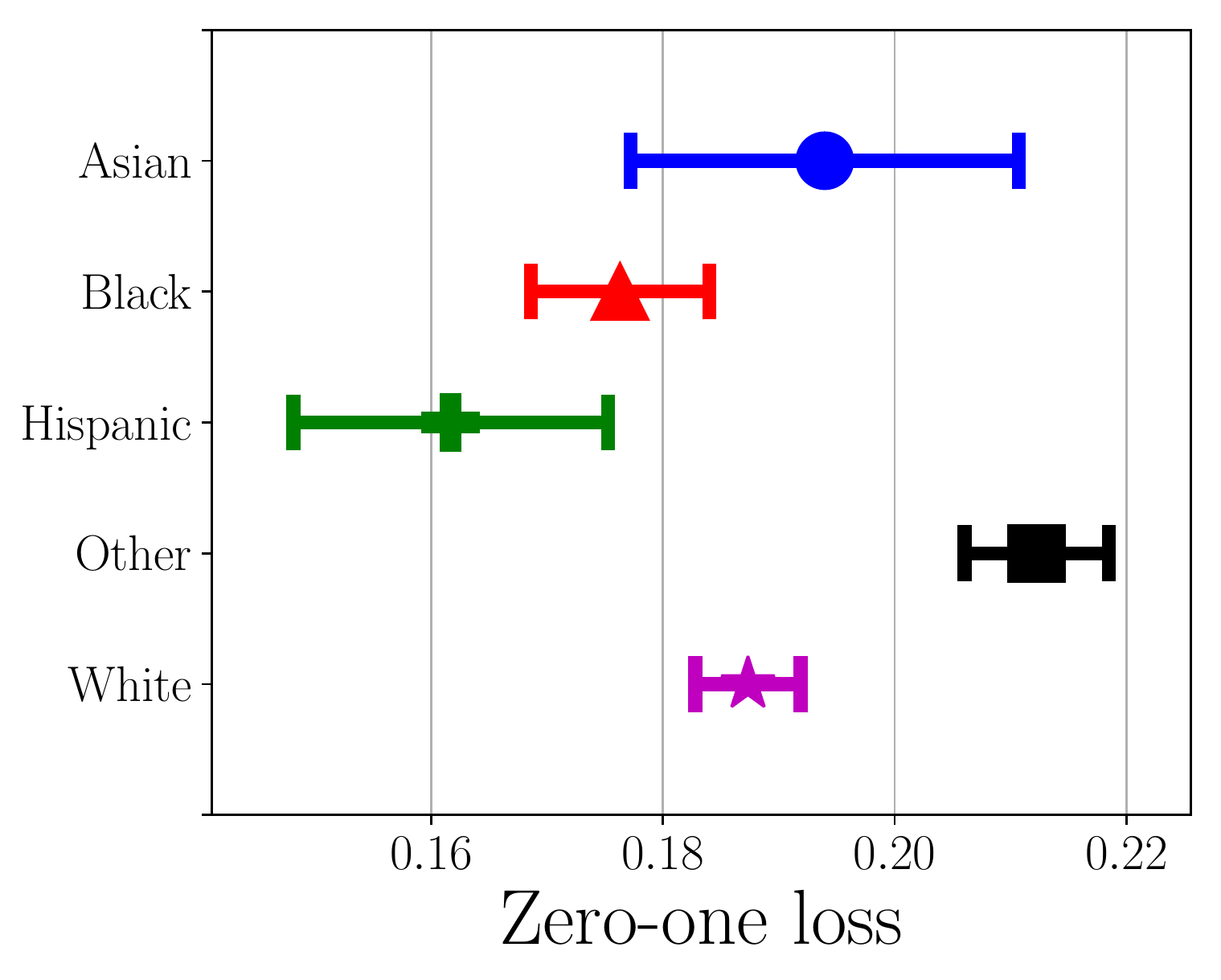}
\caption{\label{fig:notes_intervals} Using Tukey's range test, we can find the 95\%-significance level for the zero-one loss for each group over 5-fold cross validation.}
\end{subfigure}
\;
\begin{subfigure}[b]{.31\textwidth}
\centering
\includegraphics[width=1.\columnwidth]{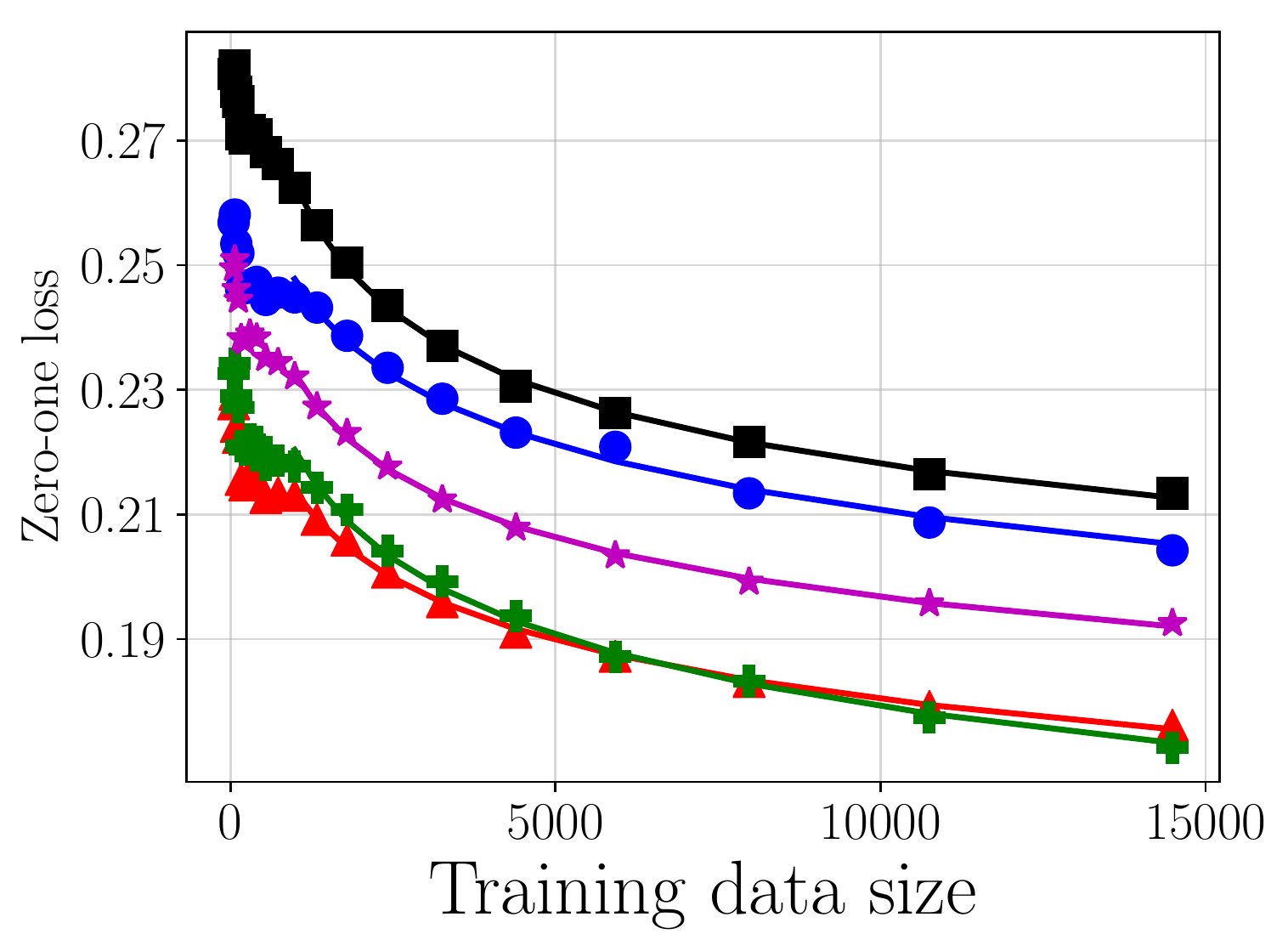}
\caption{\label{fig:notes_addtrain} As training set size increases, zero-one loss over 50 trials decreases over all groups and appears to converge to an asymptote.}
\end{subfigure}
\;
\begin{subfigure}[b]{.31\textwidth}
% \centering
\includegraphics[width=1.\textwidth]{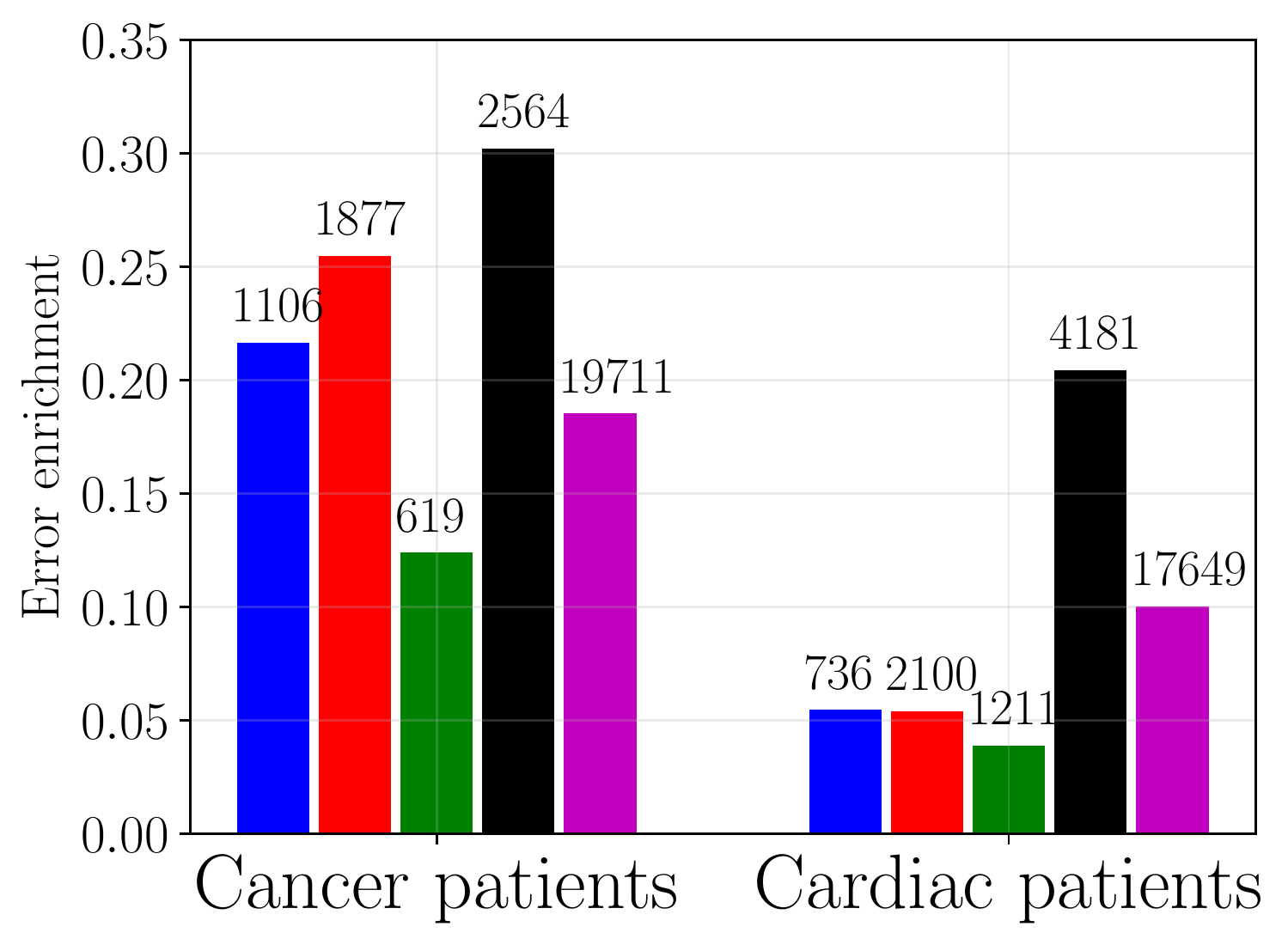}
\caption{\label{fig:topic_errors} Topic modeling reveals subpopulations with high differences in zero-one loss, for example cancer patients and cardiac patients.}
\end{subfigure}
\caption{\label{fig:exp_notes} Mortality prediction from clinical notes using logistic regression. Best viewed in color.}
\end{figure*}

Unstructured medical data such as clinical notes can reveal insights for questions like mortality prediction; however, disparities in predictive accuracy may result in discrimination of protected groups. 
Using the MIMIC-III dataset of all clinical notes from 25,879 adult patients from Beth Israel Deaconess Medical Center \citep{mimiciii}, we predict hospital mortality of patients in critical care. Fairness is studied with respect to five self-reported ethnic groups of the following proportions: Asian (2.2\%), Black (8.8\%), Hispanic (3.4\%), White (70.8\%), and Other (14.8\%). Notes were collected in the first 48 hours of an intensive care unit (ICU) stay; discharge notes were excluded. We only included patients that stayed in the ICU for more than 48 hours. We use the tf-idf statistics of the 10,000 most frequent words as features. Training a model on 50\% of the data, selecting hyper-parameters on 25\%, and testing on 25\%, we find that logistic regression with L1-regularization achieves an AUC of 0.81. The logistic regression is well-calibrated with Brier scores ranging from 0.06-0.11 across the five groups; we note better calibration is correlated with lower prediction error.

We report cost and discrimination level in terms of generalized zero-one loss~\citep{pleiss2017fairness}. Using an ANOVA test~\citep{fisher1925statistical} with  $p < 0.001$, we reject the null hypothesis that loss is the same among all five groups. To map the 95\% confidence intervals, we perform pairwise comparisons of means using Tukey's range test \citep{tukey1949comparing} across 5-fold cross-validation. As seen in Figure~\ref{fig:notes_intervals}, patients in the Other and Hispanic groups have the highest and lowest generalized zero-one loss, respectively, with relatively few overlapping intervals. Notably, the largest ethnic group (White) does not have the best accuracy, whereas smaller ethnic groups tend towards extremes. While racial groups differ in hospital mortality base rates (Table 1 in the Supplementary material), Hispanic (10.3\%) and Black (10.9\%) patients have very different error rates despite similar base rates.

To better understand the discrimination induced by our model, we explore the effect of changing training set size. To this end, we repeatedly subsample and split the data, holding out at least 20\% of the full data for testing. In Figure~\ref{fig:notes_addtrain}, we show loss averaged over 50 trials of training a logistic regression on increasingly larger training sets; estimated inverse power-law curves show good fits. We see that some pairwise differences in loss decrease with additional training data. 

Next, we identify clusters for which the difference in prediction errors between protected groups is large. We learn a topic model with $k=50$ topics generated using Latent Dirichlet Allocation \citep{blei2003latent}. Topics are concatenated into an $n\times k$ matrix $Q$ where $q_{ic}$ designates the proportion of topic $c\in[k]$ in note $i\in [n]$. Following prior work on enrichment of topics in clinical notes \citep{marlin2012unsupervised, ghassemi2014unfolding}, we estimate the probability of patient mortality $Y$ given a topic $c$ as $\hat{p}(Y|C=c) := (\sum_{i=1}^n y_i  q_{ic})/(\sum_{i=1}^n q_{ic})$ where $y_i$ is the hospital mortality of patient $i$. We compare relative error rates given protected group and topic using binary predicted mortality $\hat{y}_i$, actual mortality $y_i$, and group $a_i$ for patient $i$ through 

\[\hat{p}(\hat{Y} \neq Y \mid A = a', C=c) = \frac{\sum_{i=1}^n \mathbb{1}(y_i \neq \hat{y_i}) \mathbb{1}(a_i = a')  q_{ic} }{\sum_{i=1}^n \mathbb{1}(a_i = a') q_{ic}}\]

which follows using substitution and conditioning on $A$. These error rates were computed using a logistic regression with L1 regularization using an 80/20 train-test split over 50 trials. While many topics have consistent error rates across groups, some topics (e.g. cardiac patients or cancer patients as shown in Figure~\ref{fig:topic_errors}) have large differences in error rates across groups. We include more detailed topic descriptions in the supplementary material. Once we have identified a subpopulation with particularly high error, for example cancer patients, we can consider collecting more features or collecting more data from the same data distribution. We find that error rates differ between $0.12$ and $0.30$ across protected groups of cancer patients, and between $0.05$ and $0.20$ for cardiac patients. 

%
% BOOK REVIEW RATINGS
%
\subsection{Book review ratings}
In the supplementary material, we study prediction of book review ratings from review texts \citep{goodreads}. The protected attribute was chosen to be the gender of the author as determined from Wikipedia. In the dataset, the difference in mean-squared error $\Gamma^{\scmse}(\hY)$ has 95\%-confidence interval $0.136 \pm 0.048$ with $\mse_M = 0.224$ for reviews for male authors and $\mse_F = 0.358$. Strikingly, our findings suggest that $\Gamma^{\scmse}(\hY)$ may be completely eliminated by additional targeted sampling of the less represented gender.

%
% DISCUSSION
%
\section{Discussion}

We identify that existing approaches for reducing discrimination induced by prediction errors may be unethical or impractical to apply in settings where predictive accuracy is critical, such as in healthcare or criminal justice. As an alternative, we propose a procedure for analyzing the different sources contributing to discrimination. Decomposing well-known definitions of cost-based fairness criteria in terms of differences in bias, variance, and noise, we suggest methods for reducing each term through model choice or additional training data collection. Case studies on three real-world datasets confirm that collection of additional samples is often sufficient to improve fairness, and that existing post-hoc methods for reducing discrimination may unnecessarily sacrifice predictive accuracy when other solutions are available. 

Looking forward, we can see several avenues for future research. In this work, we argue that identifying clusters or subpopulations with high predictive disparity would allow for more targeted ways to reduce discrimination. We encourage future research to dig deeper into the question of local or context-specific unfairness in general, and into algorithms for addressing it. Additionally, extending our analysis to intersectional fairness~\citep{buolamwini2018gender, hebert2017calibration}, e.g. looking at both gender and race or all subdivisions, would provide more nuanced grappling with unfairness. Finally, additional data collection to improve the model may cause unexpected delayed impacts~\citep{liu2018delayed} and negative feedback loops~\citep{ensign17} as a result of distributional shifts in the data. More broadly, we believe that the study of fairness in  non-stationary populations is an interesting direction to pursue.

\section*{Acknowledgements}
The authors would like to thank Yoni Halpern and Hunter Lang for helpful comments, and Zeshan Hussain for clinical guidance. This work was partially supported by Office of Naval Research Award No. N00014-17-1-2791 and NSF CAREER award \#1350965.

\bibliography{bib}
\bibliographystyle{icml2018}

\clearpage
\appendix
\section{Testing for significant discrimination}
\label{sec:disc_test}
In general, neither $\Gamma$ nor $\olGamma$ can be computed exactly, as the expectations $\gamma_a = \E_p[L(Y, \hat{Y}) \mid A=a]$ and $\olgamma$, for $a\in \cA$ are known only approximately  through a set of samples $S = \{(x_i, a_i, y_i)\}_{i=1}^m \sim p^m$ drawn from the (possibly class-conditional) population $p$. The Monte Carlo estimate, 
$$
\gamma^S_a(\hat{Y}) = \frac{1}{m_a} \sum_{i=1}^{m} L(y_i, \hat{y}_i) \mathds{1}[a_i = a]~,
$$
with $m_a = \sum_{i=1}^m \mathds{1}[a_i = a]$, 
may be used to form an estimate $\Gamma^S(\hat{Y}) = |\gamma^S_0(\hat{Y}) - \gamma^S_1(\hat{Y})|$. By the central limit theorem, for sufficiently large $m$, $\gamma^S_a(\hat{Y}) \sim \cN(\mu_a, \sigma_a^2/m_a)$ and $(\gamma^S_0 - \gamma^S_1) \sim \cN(\mu_0 - \mu_1, \sigma_0^2/m_0 + \sigma_1^2/m_1)$. As a result, the significance of $\Gamma^S(\hat{Y})$ can be tested with a two-tailed z-test or using the test of \citet{woodworth2017learning}. If sample sizes are small and the target binary, more appropriate tests are available~\citep{brown2001interval}.
In addition, we will often want to compare the discrimination levels $\Gamma(\hat{Y}), \Gamma(\hat{Y}')$ of predictors $\hat{Y}, \hat{Y}'$, resulting from different learning algorithms, models, or sets of observed variables. The random variable $|\Gamma^S(\hat{Y}) - \Gamma^S(\hat{Y}')|$ is not Normal distributed, but is an absolute difference of folded-normal variables. However, for any $\alpha \in \{-1,1\}$, $Z_{\alpha} := \alpha(\gamma_0^S(\hat{Y}) - \gamma_1^S(\hat{Y})) - (\gamma_0^S(\hat{Y}') - \gamma_1^S(\hat{Y}'))$ is Normal distributed. Further, by enumerating the signs of $(\gamma_0^S(\hat{Y}) - \gamma_1^S(\hat{Y}))$ and $(\gamma_0^S(\hat{Y}') - \gamma_1^S(\hat{Y}'))$, we can show that $|\Gamma^S - {\Gamma^S}'| = \min_{\alpha \in \{-1,1\}} |Z_{\alpha}|$. As a result, to reject the null hypothesis $H_0 : \Gamma = \Gamma'$, we require that the observed values of both $Z_{-1}$ and $Z_{1}$ are unlikely under $H_0$ at given significance.

\section{Additional experimental details}

\subsection{Datasets}
\label{app:datasets}
\begin{itemize}
\item Adult Income Dataset~\citep{lichman2013uci}. The dataset has 32,561
  instances. The target variable indicates whether
  or not income is larger than 50K dollars, and the
  sensitive feature is Gender. Each data object is
  described by 14 attributes which include 8 categorical
  and 6 numerical attributes. We quantize the categorical attributes into binary features and keep the continuous attributes, which results in 105 features for prediction. We note the label imbalance as 30\% of male adults have income over 50K whereas only 10\% of female adults have income over 50K. Additionally 24\% of all adults have salary over 50K, and the dataset has 33\% women and 67\% men.
\item Goodreads reviews~\cite{goodreads}, only included in the supplemental materials. The dataset was collected from Oct 12, 2017 to Oct 21, 2017 and has 13,244 reviews. The target variable is the rating of the review, and the sensitive feature is the gender of the author. Genders were gathered by querying Wikipedia and using pronoun inference, and the dataset is a subset of the original Goodreads dataset because it only includes reviews about the top 100 most popular authors. Each datum consists of the review text, vectorized using Tf-Idf. The review scores occurred with counts 578, 2606, 4544, 5516 for scores 1,3,4, and 5 respectively. Books by women authors and men authors had average scores of 4.088 and 4.092 respectively.
\item MIMIC-III dataset \citep{mimiciii}. The dataset includes 25,879 adult patients admitted to the intensive care unit of the Beth Israel Deaconess Medical Center in downtown Boston. Clinical notes from the first 48 hours are used to predict hospital mortality after 48 hours. Of all adult patients, 13.8\% patients died in the hospital. We are interested in the difference in performance between the five self-reported ethnic groups and following data sizes and hospital mortality rates.

\begin{table}[h!]
\centering
 \begin{tabular}{c c c c} 
 Race & \# patients & \% total & Hospital Mortality \\ 
 \hline
 \hline
Asian & 583 & 2.3 & 14.2 \\
Black & 2,327 & 9.0 & 10.9 \\
Hispanic & 832 & 3.2 & 10.3 \\
Other & 3,761 & 14.5 & 18.4 \\
White & 18,377 & 71.0 & 13.4 \\
 \end{tabular}
 \caption{Summary statistics of clinical notes dataset}
\end{table}
\end{itemize}

\subsection{Synthetic experiments}
To illustrate the effect of training set size and model choice, and the validity of the power-law learning curve assumption, we conduct a small synthetic experiment in which $p(A=1) = 0.3$ and $X \sim \cN(\mu_A, \sigma_A^2)$ with $\mu_0 = 0, \mu_1 = 1, \sigma_0 = 1, \sigma_1 = 2$. The outcome is a quadratic function with heteroskedastic noise, $Y = 2X^2 - 2X + .1 + \epsilon X^2$, with $\epsilon \sim \cN(0,1)$. We fit decision tree, random forest and ridge regressors of the outcome $Y$ to $X$ using default parameters in the implementation in scikit-learn~\citep{scikit-learn}, but limiting the decision tree to depth $T\leq 4$. The size of the training set is varied exponentially between $2^5$ and $2^{17}$ samples, and at each size, trees are fit 200 times. In Figure~\ref{fig:learning_curves}, we show the resulting learning curves $\olgamma_0(\hat{Y},n)$ and $\olgamma_1(\hat{Y},n)$ as well as fits of Pow3 curves to them. Shown in dotted lines are extrapolations of learning curves from different sample sizes, illustrating the difficulty of estimating the intercepts $\delta_a$ and the Bayes error with high accuracy. 

\begin{figure}[tbp!]
\centering
\begin{subfigure}[b]{.47\textwidth}
\centering
\includegraphics[width=1\textwidth]{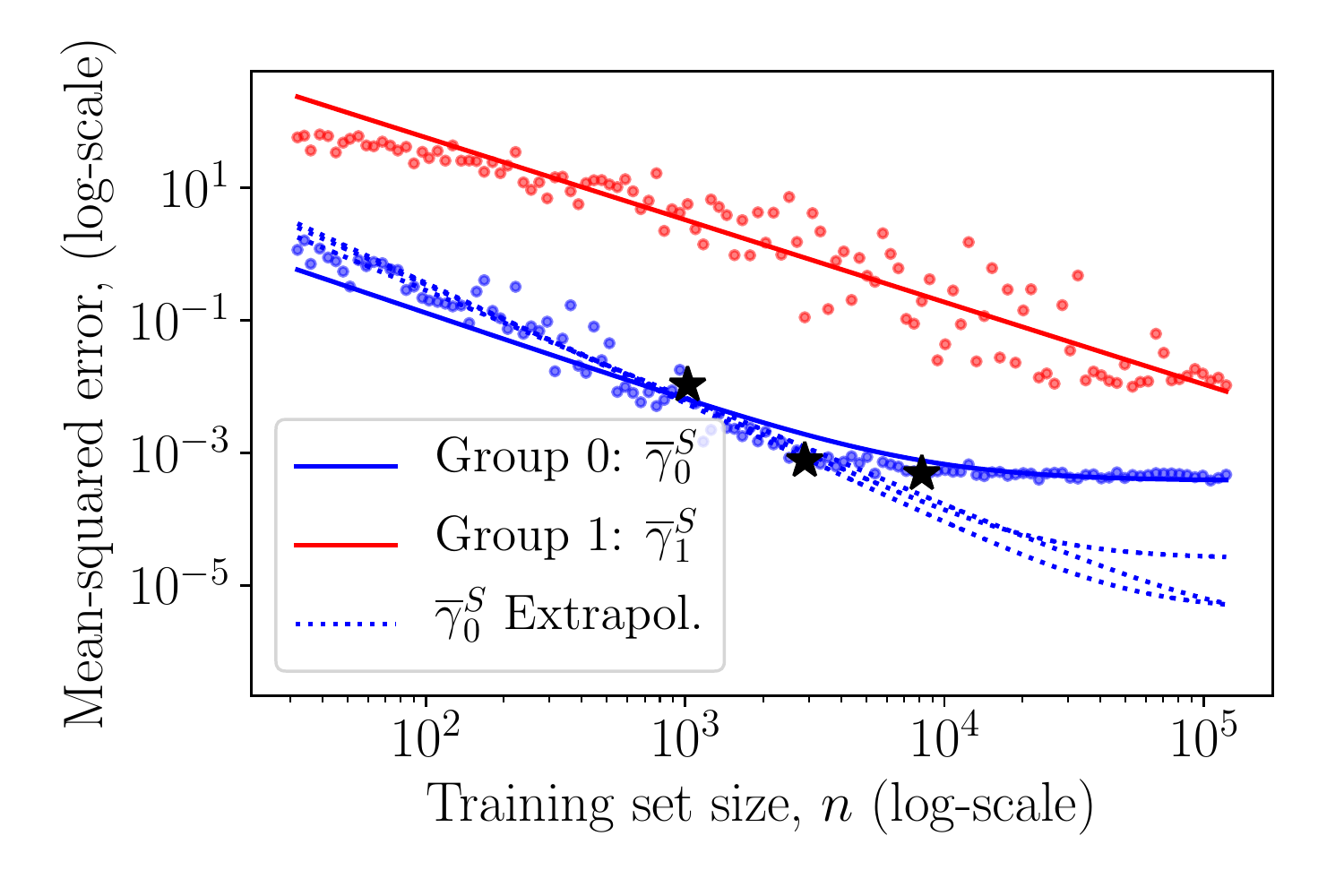}%
\vspace{-.1in}
\caption{Learning curves, $\olgamma_0, \olgamma_1$ for random forest}
\end{subfigure}
\;
\begin{subfigure}[b]{.47\textwidth}
\centering
\includegraphics[width=1\textwidth]{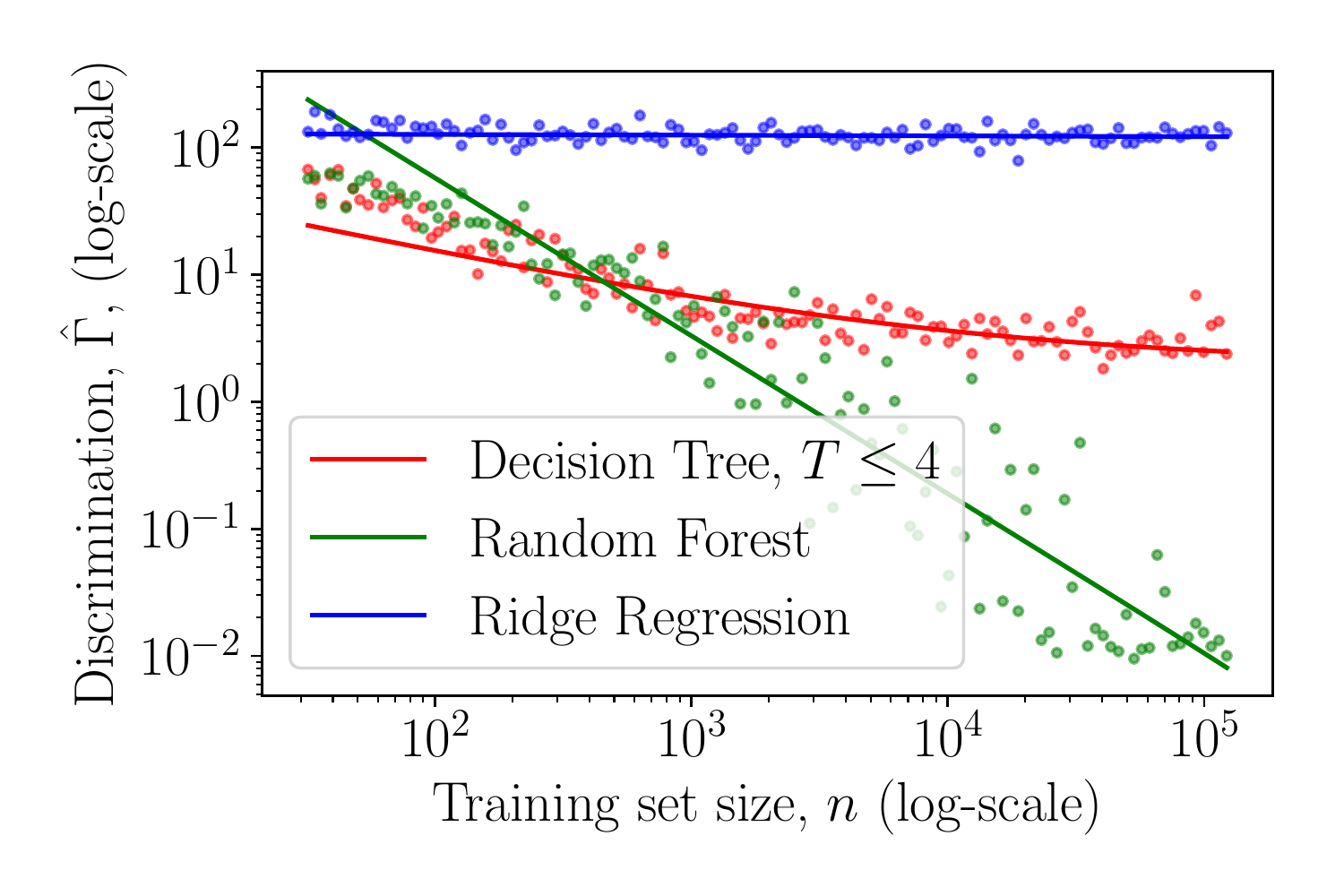}%
\vspace{-.1in}
\caption{Discrimination, $\olGamma = |\olgamma_0 - \olgamma_1|$ for various models}
\end{subfigure}
\caption{\label{fig:learning_curves}Inverse power-laws (Pow3) fit to generalization error as a function of training set size on synthetic data. Dotted lines are extrapolations from sample sizes indicated by black stars. This illustrates the difficulty of estimating the Bayes error through extrapolation, here at $\overline{N}_0 = 3\cdot 10^{-4}$ and $\overline{N}_1 = 7\cdot 10^{-3}$ respectively.}
\end{figure}

\subsection{Book review ratings}
% ---------
% BOOKS
% ---------
\begin{figure*}[tbph!]
\centering
\begin{subfigure}[b]{.45\textwidth}
\centering
\includegraphics[width=1.\textwidth]{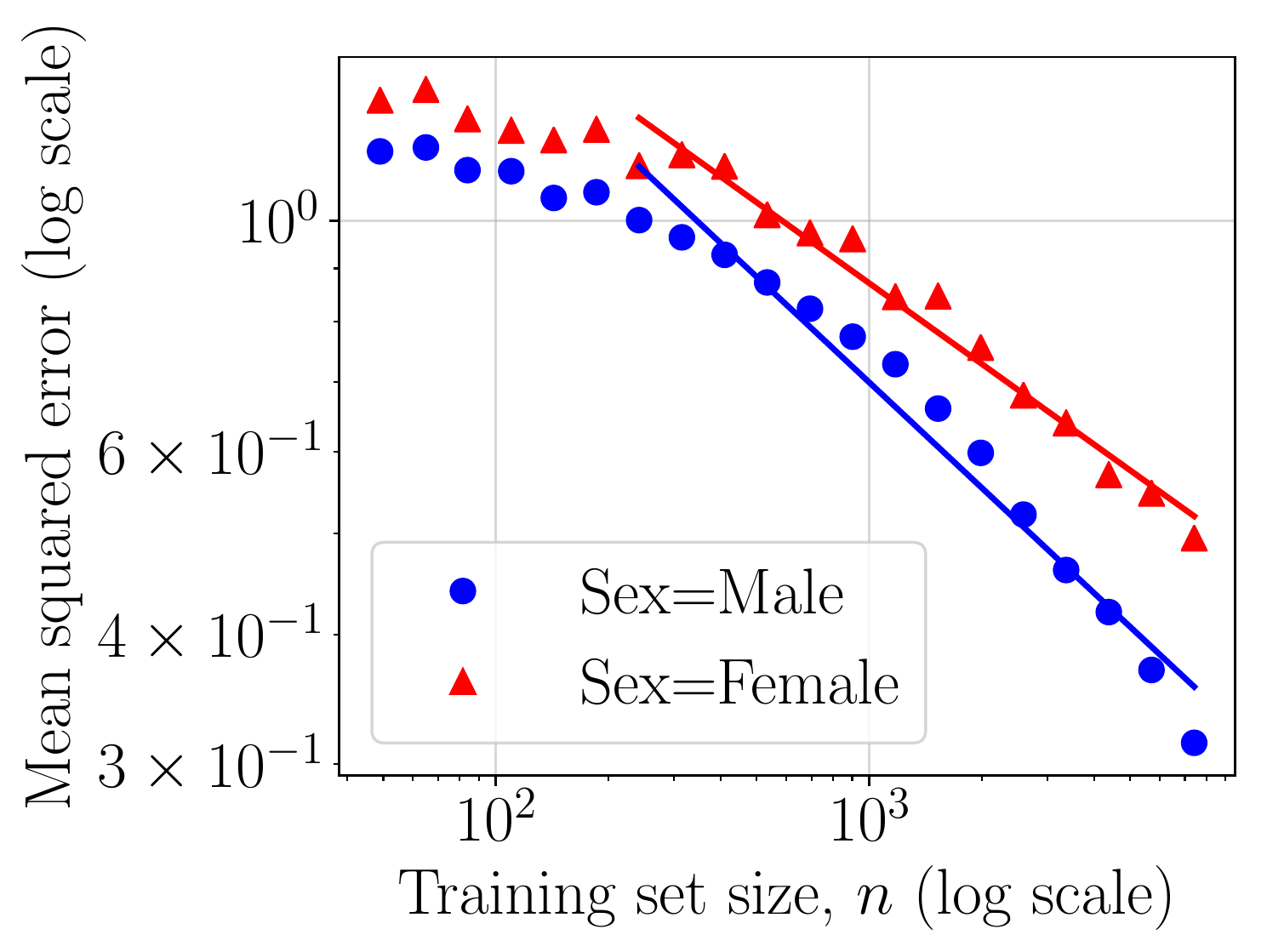}
\caption{\label{fig:books} As training set size increases for random forest, $\mse$ decreases but maintains difference between groups. Intercepts from fitted power-laws show no difference in noise.}
\end{subfigure}
\;
\begin{subfigure}[b]{.45\textwidth}
\centering
\includegraphics[width=1.\columnwidth]{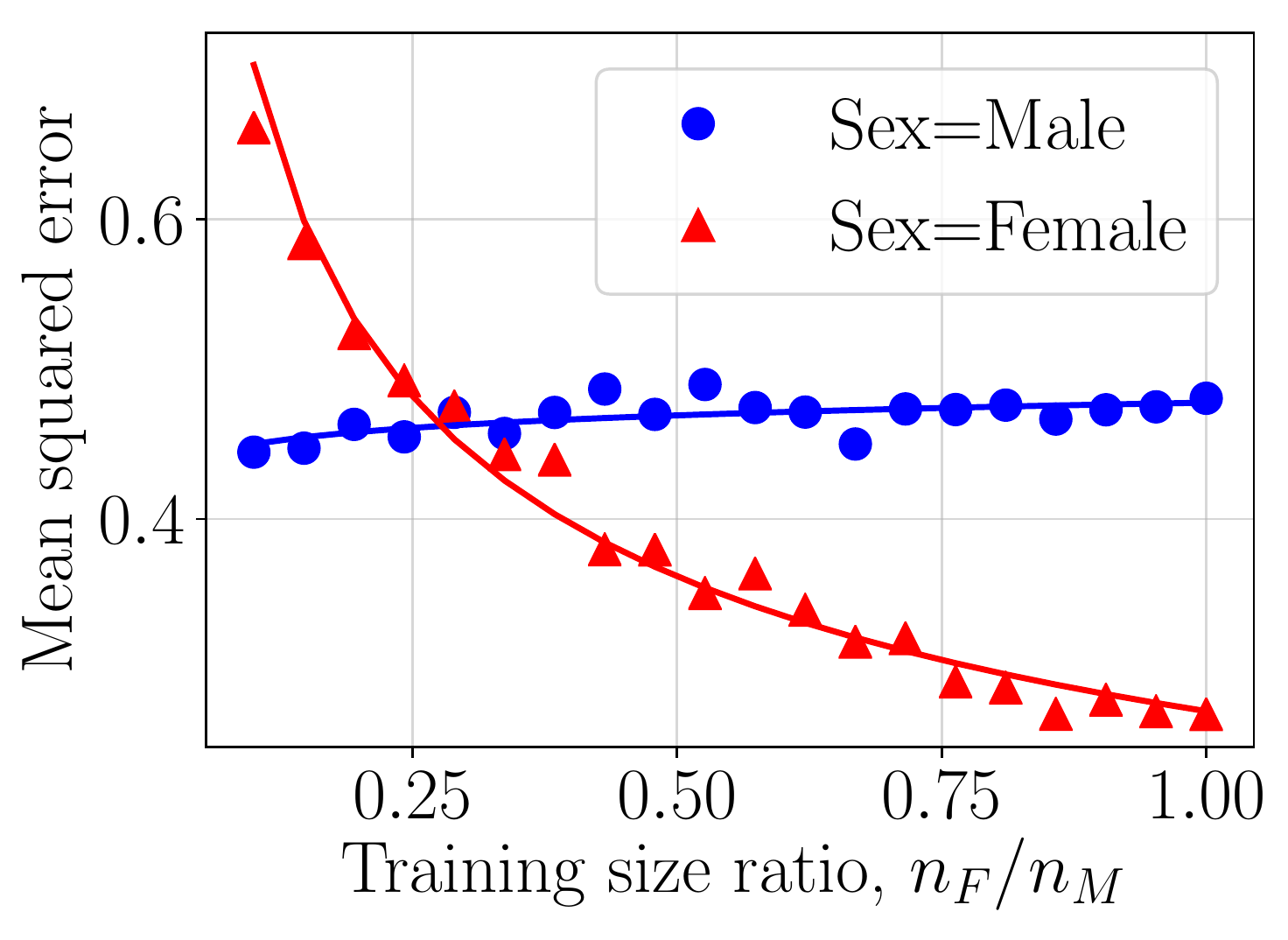}
\caption{\label{fig:books_add_onesided} Holding number of reviews for male authors $n_M$ steady and varying number of reviews for female authors $n_F$, we can achieve higher $\mse$ for one group than with the full dataset. }
\end{subfigure}
\caption{\label{fig:exp_books} Goodreads dataset for book rating prediction. Adding training data decreases overall mean squared error ($\mse$) for both groups while adding training data to only one group has a much bigger impact on reducing $\overline{\Gamma}$. Increasing the number of features reduces $\mse$ but does not reduce $\overline{\Gamma}$.}
\end{figure*}

Sentiment and rating prediction from text reveal quantitative insights from unstructured data; however deficiencies in algorithmic prediction may incorrectly represent populations. Using a dataset of 13,244 reviews collected from Goodreads \citep{goodreads} with inferred author sex scraped from Wikipedia, we seek to predict the review rating based on the review text. We use as features the Tf-Idf statistics of the 5000 most frequent words. Our protected attribute is gender of the author of the book, and the target attribute is the rating (1-5) of the review. The data is heavily imbalanced, with 18\% reviews about female authors versus 82\% reviews about male authors.

We observe statistically significant levels of discrimination with respect to mean squared error ($\mse$) with linear regression, decision trees and random forests. Using a random forest and training on 80\% of the dataset and testing on 20\%, we find that our $\Gamma^{\scmse}(\hY)$ has 95\%-confidence interval $0.136 \pm 0.048$ with $\mse_M = 0.224$ for reviews for male authors and $\mse_F = 0.358$ for reviews for female authors using a difference in means statistical test. Results were found after hyperparameter turning for each training set size and taking an average over 50 trials. We observe similar patterns with linear regression and decision trees.

To estimate the impact of additional training data, we evaluate the effect of varying training set size $n$ on predictive performance and discrimination. Through repeated sample spitting, we train a random forest on increasing training set sizes, reserving at least 20\% of the dataset for testing. In Figure~\ref{fig:books}, additional training data lowers $\mse_F$ and $\mse_M$, fitting an inverse power-law. Based on the intercept terms of the extrapolated power-laws ($\delta_M = 0.0011$ for reviews with male authors and $\delta_F = 0.0013$ for reviews with female authors), we may expect that $\overline{\Gamma}$ can be explained more by differences in bias and variance than by noise since our estimated difference in noise $|\delta_F - \delta_M| \approx 0$.

In order to further measure the effect of collecting more samples, we analyze a one-sized increase in training data. Because of the initial skew of author genders in the dataset, we vary the number of reviews for female authors, creating a shift in populations in the training data. We fix the training set size of reviews for male authors at $n_M = 1939$, which represents the size of the full data for female authors $N_F$, reserving 20\% of the dataset as test data. We then vary the training data size for female authors $n_F$ such that the ratio $n_F / n_M$ varies evenly between 0.1 to 1.0. Using a linear regression in Figure~\ref{fig:books_add_onesided}, we see that as the ratio $n_F / n_M$ increases, $\mse_F$ decreases far below $\mse_M$ and far below our best reported $\mse$ of the random forest on the full dataset. This suggests that shifting the data ratio and collecting more data for the under-represented group can adapt our model to reduce discrimination.

\subsection{Clinical notes}
% ---------
% Clinical notes
% ---------

% \subsection{Topic modeling in clinical notes}

Here we include additional details about topic modeling. Topics were sampled using Markov Chain Monte Carlo after 2,500 iterations. We present the topics with highest and lowest variance in error rates among groups in Table~\ref{tab:topics}. Error rates were computed using a logistic regression with L1 regularization over 10,000 TF-IDF features using 80/20 training and testing data split over 50 trials. Based on the most representative words for each topic, we can infer topic descriptions, for example cancer patients for topic 48 and cardiac patients for topic 45.

% \begin{landscape}
\begin{table}[h!]
\centering
 \begin{tabular}{c L c c c c c} 
 Topic & Top words &  Asian & Black & Hispanic & Other & White \\ 
 \hline \hline
31 & no(t pain present normal edema tube history pulse absent left respiratory monitor & 5.9 & 8.4 & 17.6 & 30.8 & 11.1\\
\hline
17 & hospital lymphoma continue s/p unit bmt thrombocytopenia line rash & 34.3 & 13.6 & 34.9 & 30.2 & 26.0\\
\hline
43 & bowel abdominal abd abdomen surgery s/p small pain obstruction fluid ngt & 16.6 & 11.8 & 5.7 & 26.8 & 13.2\\
\hline
45 & artery carotid aneurysm left identifier numeric vertebral internal clip & 5.4 & 5.3 & 3.8 & 20.4 & 10.0\\
\hline
48 & mass cancer metastatic lung tumor patient cell left malignant breast hospital & 21.6 & 25.4 & 12.3 & 30.2 & 18.5\\
\hline \hline
1 & neo gtt pain resp neuro wean clear plan insulin good & 3.3 & 1.8 & 1.6 & 3.6 & 2.7\\
\hline
2 & assessment insulin mg/dl plan pain meq/l mmhg chest cabg action & 0.3 & 0.6 & 0.9 & 3.6 & 2.2\\
\hline
0 & chest reason tube clip left artery s/p pneumothorax cabg pulmonary  & 3.2 & 5.5 & 2.5 & 5.6 & 4.0\\
\hline
25 & c/o pain clear denies oriented sats plan alert stable monitor & 7.3 & 3.9 & 5.9 & 8.2 & 6.5\\
\hline
47 & pacer pacemaker icd s/p paced rhythm ccu amiodarone cardiac & 8.2 & 9.1 & 8.3 & 13.8 & 10.1\\
\hline
 \end{tabular}
 \caption{Top and bottom 5 topics (of 50) based on variance in error rates of groups. Error rates by group and topic $p(\hat{Y} \neq Y | K, A)$ are reported in percentages.}
 \label{tab:topics}
\end{table}
% \end{landscape}

We identified patients with notes corresponding to topic 48, corresponding to cancer, as a subpopulation with large differences in errors between groups. By varying the training size while saving 20\% of the data for testing, we estimate that more data would not be beneficial for decreasing error (see Figure~\ref{fig:cancer}). The mean over 50 trials is reported with hyperparameters chosen for each training size. Instead, we recommend collecting more features (e.g. structured data from lab results, more detailed patient history) as a way of improving error for this subpopulation.

Furthermore, we compute the 95\% confidence intervals for false positive and false negative rates for a logistic regression with L1 regularization in Figure~\ref{fig:notes_fn} and Figure~\ref{fig:notes_fp}.

\begin{figure*}[tbp!]
\includegraphics[width=0.9\textwidth]{figs/legend.pdf}
\centering 
\begin{subfigure}[b]{.31\textwidth}
\centering
\includegraphics[width=1\columnwidth]{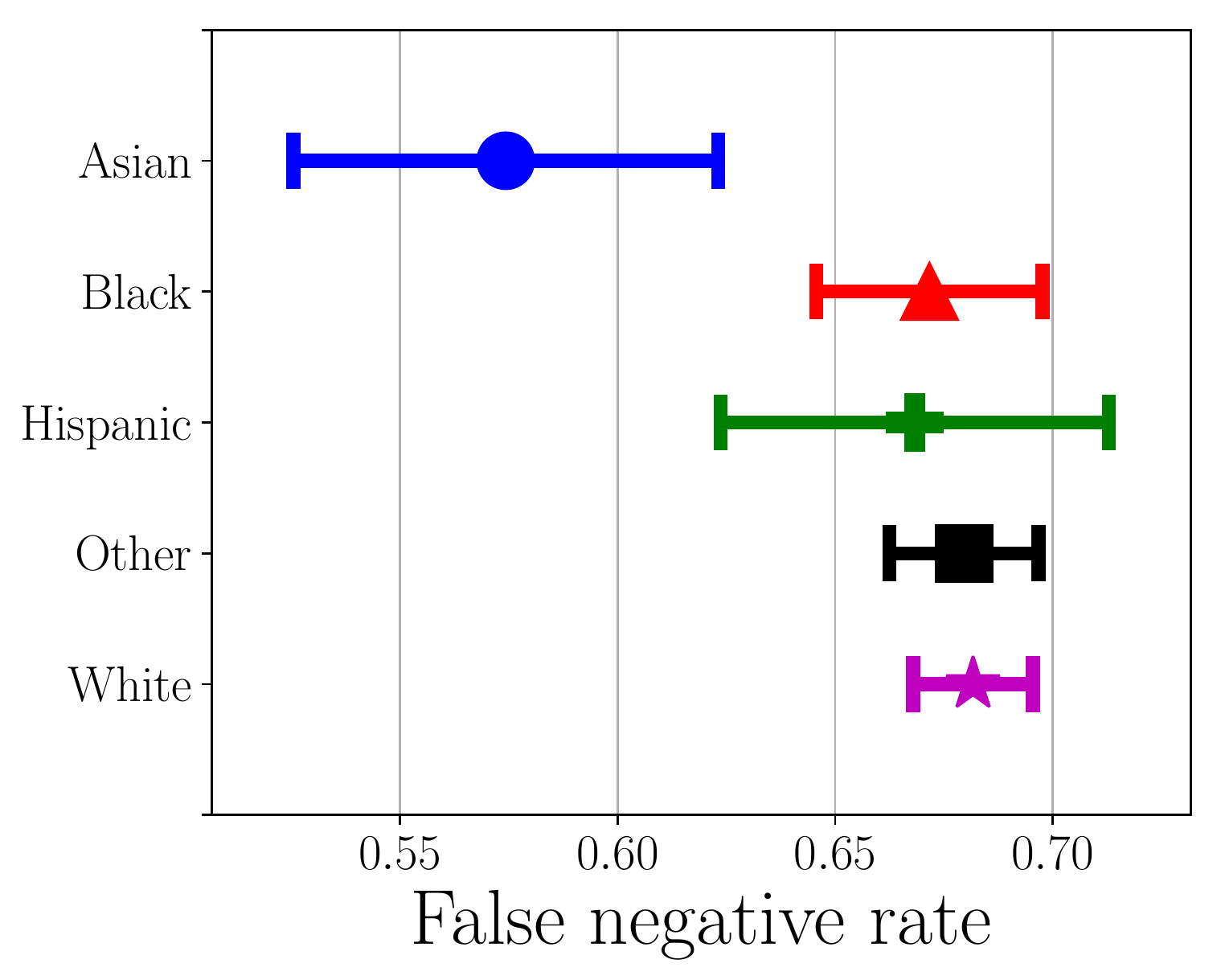}
\caption{\label{fig:notes_fn} The false negative rates for logistic regression with L1 regularization do not differ across five ethnic groups, shown by the overlapping 95\%-confidence intervals intervals, except for Asian patients.}
\end{subfigure}
\;
\begin{subfigure}[b]{.31\textwidth}
\centering
\includegraphics[width=1.\columnwidth]{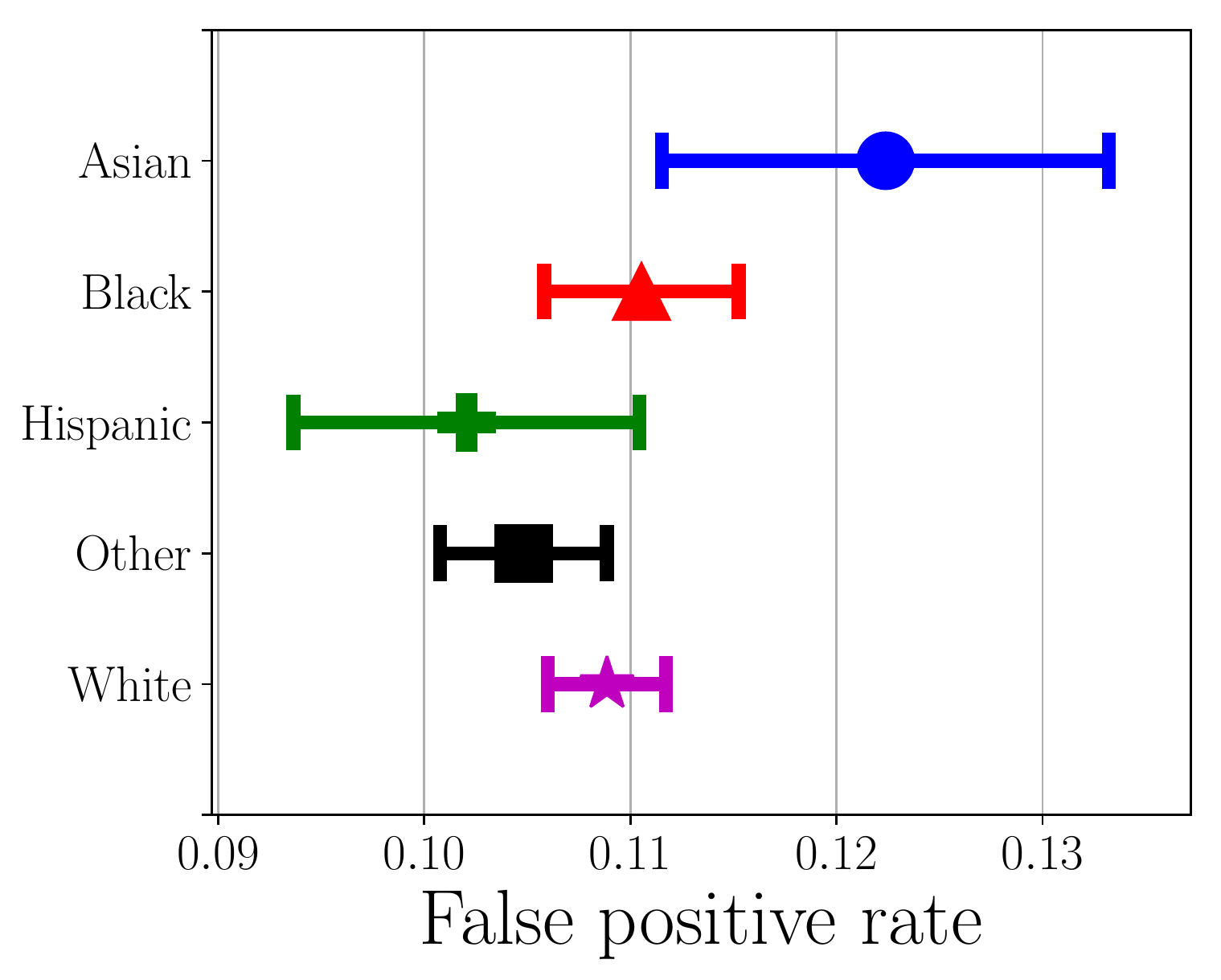}
\caption{\label{fig:notes_fp} The false positive rates also does not differ much across groups with many overlapping intervals. Note that Asian patients have high false positive rate but low false negative rates.}
\end{subfigure}
\;
\begin{subfigure}[b]{.31\textwidth}
% \centering
\includegraphics[width=1.\textwidth]{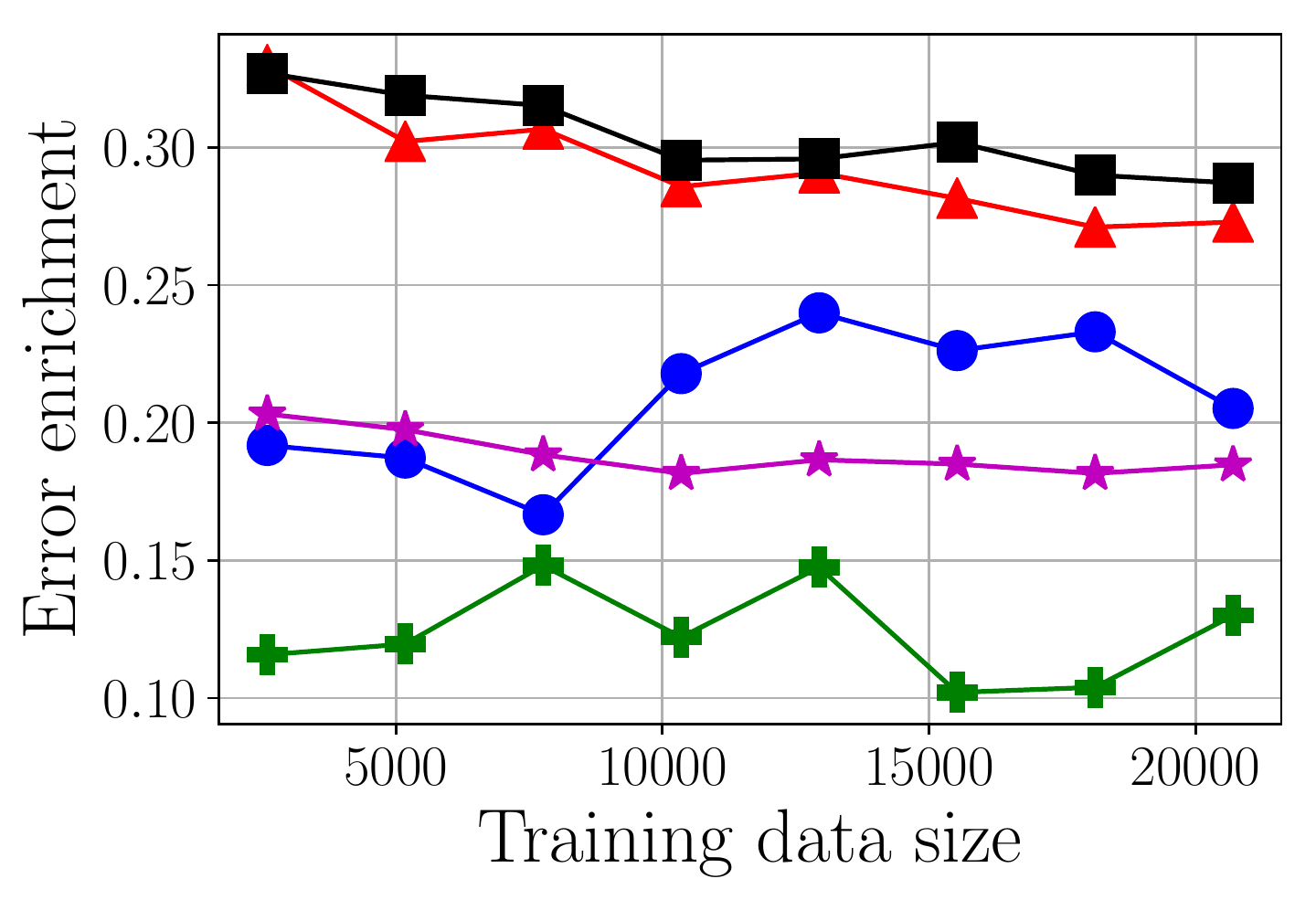}
\caption{\label{fig:cancer} Adding training data size on error enrichment for cancer (topic 48) does not necessarily reduce error for all groups. This may suggest we should focus on collecting more features instead.}
% \begin{tabular}{cccccc}
% Topic & Asian & Black & Hispanic & Other & White \\
% \hline
% \hline
% % Blood disorders & 13.2 & 3 & 1 & 3 & 2 \\
% \end{tabular}
\end{subfigure}
\caption{\label{fig:notes} Additional clinical notes experiments highlight the differences in false positive and false negative rates. We also examine the effect of training size on cancer patients in the dataset.}
\end{figure*}

\section{Exploring model choice}
If a difference in bias is the dominating source of discrimination between groups, changing the class of models under consideration could have a large impact on discrimination.Consider for example Figure 1c in which the true outcome has higher complexity in regions where one protected group is more densely distributed than the other. Increasing model capacity in such cases, or exploring other model classes of similar capacity, may reduce as long as the bias-variance trade-off is beneficial. Bias is not identifiable in general, as this requires estimation or bounding of noise components $N_a$, or an assumption that they are equal, $\overline{N}_0 = \overline{N}_1$, or negligible, $\overline{N}_a \approx 0$. However, as noise is in-dependent of model choice, a difference in bias of different models is identifiable even if the noise is not known, provided that the variance is estimated. With $\Delta \overline{B} = \overline{B}_0 - \overline{B}_1$, and $\Delta \overline{V} = \overline{V}_0 - \overline{V}_1$, and $\hat{Y}, \hat{Y}'$, two predictors for comparison, we may test the hypothesis $H_0: \Delta \overline{B}(\hat{Y}) + \Delta \overline{V}(\hat{Y}) = \Delta \overline{B}(\hat{Y}') + \Delta \overline{V}(\hat{Y}')$.

\section{Regression with homoskedastic noise}
By definition of $\olN$, we can state the following result.
\begin{thmprop}
Homoskedastic noise, i.e. $\forall x\in \cX, a\in \cA : N(x,a) = N$, does not contribute to discrimination level $\olGamma$ under the squared loss $L(y, y') = (y - y')^2$.
\end{thmprop}
\begin{proof}
Under the squared loss, $\forall a: \olN_a = \E_X[N(X,a)] = N$, as $c_n(x,a) = 1$. 
\end{proof}
In contrast, for the zero-one loss and class-specific variants, the expected noise terms $\overline{N}_a$ do not cancel, as they depend on the factor $c_n(x, a)$.

\section{Bias-variance decomposition. Proof of Theorem~1.}

\begin{thmapplem}[Squared loss and zero-one loss]\label{lem:appzo2}
The following claim holds for both:\\
a) $L(y,y') = \mathds[y \neq y']$ the zero-one loss with $c_1(x,a)  = 2\mathbb{E}[\mathds{1}[\hat{Y}_D(x,a) = \hy_*(x,a)]]-1$ and $c_2(x,a) = \{1, \mbox{ if } \hy^*(x,a) = \hy^m(x,a); -1 \mbox{ otherwise}\}$,\\ 
b) a) $L(y,y') = (y-y')^2$ the squared loss with $c_1(x,a) = c_2(x,a) = 1$.
\begin{align*}
\E[L(Y,\hY_D)\mid X=x, A=a] & = c_1(x,a)\E[L(y,\hY^*)\mid x, a] \\
& + L(\hy^m(x,a), \hy^*(x,a)) + c_2\E[L(\hy^m(x,a), \hY_D) \mid x, a]~.
\end{align*}
\end{thmapplem}
\begin{proof}See \citet{domingos2000unified}.\end{proof}

\begin{thmapplem}[Class-specific zero-one loss]\label{lem:appclass}
With $L(y,y') = \mathds[y \neq y']$ the zero-one loss, it holds with $c_1(x,a)  = 2\mathbb{E}[\mathds{1}[\hat{Y}_D(x,a) = \hy_*(x,a)]]-1$ and $c_2(x,a) = \{1, \mbox{ if } \hy^*(x,a) = \hy^m(x,a); -1 \mbox{ otherwise}\}$
\begin{align*}
\forall y\in \{0,1\} : \E[L(y,\hY_D)\mid X=x, A=a] = \\
c_1(x,a)L(y,\hY^*) + L(\hy^m(x,a), \hy^*(x,a)) + c_2\E[L(\hy^m(x,a), \hY_D) \mid x, a]~.
\end{align*}
\end{thmapplem}
\begin{proof}
We begin by showing that $L(y, \hY_D(x,a)) = L(\hy^*(x,a), \hY_D(x,a)) + c_0(x,a)L(y, \hy^*(x,a))$ with $c_0(x,a) = \{+1, \mbox{ if } \hy^*(x,a) = \hY_D(x,a); -1, \mbox{otherwise}\}$.
\begin{align*}
L(y, \hY_D) - L(\hy^*(x,a), \hY_D(x,a)) + c_0(x,a)L(y, \hy^*(x,a)) \\
= \left\{\begin{array}{ll}
0,& \mbox{ if } \hY_D(x,a) = \hy^*(x,a) = 0\\
-1-c_0(x,a),& \mbox{ if } \hY_D(x,a) = 0, \hy^*(x,a) = 1\\
0,& \mbox{ if }  \hY_D(x,a) = 1, \hy^*(x,a) = 0\\
1-c_0(x,a),& \mbox{ if }  \hY_D(x,a) = \hy^*(x,a) = 1\\
\end{array}\right.
\end{align*}
As the above should be zero for all options, this implies that $c_0 = 2*\mathds{1}[\hY_D(x,a) = \hy^*(x,a)]-1$.

We now show that,
$$\E[L(\hy^*(x,a), Y_d)\mid x,a] = L(\hy^*(x,a), \hy^m(x,a)) + c_2(x,a)\E[L(\hy^m(x,a), \hY)\mid x,a]~.$$
We have that if $\hy^m(x,a) \neq \hy^*(x,a)$,
\begin{align*}
\E[L(\hy^*(x,a), \hY_D)\mid x,a] & = p(\hy^*(x,a) \neq \hY_D \mid x,a) = 1-p(\hy^*(x,a) = \hY_D \mid x,a) \\
& = 1-p(\hy^m(x,a) = \hY_D \mid x,a) = 1-\E[L(\hy^m(x,a), \hY_D) \mid x,a] \\
& = L(\hy^*(x,a), \hy^m(x,a))-\E[L(\hy^m(x,a), \hY_D) \mid x,a] \\
& =  L(\hy^*(x,a), \hy^m(x,a)) + c_2(x,a) \E[L(\hy^m(x,a), \hY_D) \mid x,a]~.
\end{align*}
A similar calculation for the case where $\hy^m(x,a) = \hy^*(x,a)$ yields the claim.

Finally, We have that 
\begin{align*}
\E[L(y, \hY_D)] &= \E[L(\hy^*(x,a), \hY_D) + c_0(x,a)L(y, \hy^*(x,a)) \mid x,a] \\
& = \E[L(\hy^*(x,a), \hY_D) \mid x,a] + \E[c_0(x,a) \mid x,a] L(y, \hy^*(x,a)) \\
& = L(\hy^*(x,a), \hy^m(x,a)) + c_2(x,a) \E[L(\hy^m(x,a), \hY_D) \mid x,a] \\
& + \E[c_0(x,a) \mid x,a] L(y, \hy^*(x,a)) 
\end{align*}
which gives us our result. 
\end{proof}

Since datasets are drawn independently of the protected attribute $A$, 
\begin{align*}
\olgamma_a(\hY) & = \E_D[\E_{X,Y}[L(Y,\hY_D)\mid D, A=a] \mid A=a] \\
& = \E_X[\E_{D,Y}[L(Y,\hY_D)\mid X, A=a]\mid A = a] \\
& = \E_X[B(\hat{Y},X,a) + c_2(X,a)V(\hat{Y},X,a) + c_1(X,a)N(X,a) \mid A=a]~,
\end{align*}
and an analogous results hold for class-specific losses, 
Theorem~1 follows from lemmas~\ref{lem:appzo2}--\ref{lem:appclass}.

\section{Difference between power law curves}
\label{app:pow3}
Let $f(x) = ax^{-b}+c$ and $g(x) = dx^{-e}+h$. Then $d(x) = f(x) - g(x)$ has at most 2 local minima. We see this by re-writing $d(x)$
$$
d(x) = ax^{-b}+\tilde{c} - dx^{-e}
$$
and so
$$
d'(x) = (-b)ax^{-b-1}+dex^{-e-1}
$$
Setting the derivative to zero,
$$
(-b)ax^{-b-1}+dex^{-e-1} = 0
$$
$$
x^{b-e} = \frac{ba}{de}
$$
which has a unique positive root
$$
x = (\frac{ba}{de})^{\frac{1}{b-e}}~.
$$
Since $f(x)$ has a single critical point (for $x>0$), $f(x)$ can switch signs at most twice. The curves $f(x) = \frac{100}{x^2}+1$ and $g(x) = \frac{50}{x}$ intersect twice on $x \in [0,\infty]$. If $b=e$, $d(x)$ has a single zero,
$$
d(x) = (a-d)x^{-b}+\tilde{c} = 0
$$
yields
$$
x = (\frac{\tilde{c}}{d-a})^{\frac{1}{-b}}~.
$$

\end{document}

%% file: header.tex
\newtheorem{thmprop}{Proposition}
\newtheorem{thmthm}{Theorem}

\newtheorem{thmasmp}{Assumption}

\newtheorem{thmapplem}{Lemma}

\theoremstyle{definition}
\newtheorem{thmdef}{Definition}

\newenvironment{thmproofsketch}[1][Proof sketch]{\begin{trivlist}
\item[\hskip \labelsep {\textit{#1.}}]}{\end{trivlist}}

\def\E{\mathbb{E}}
\def\cA{\mathcal{A}}
\def\cN{\mathcal{N}}

\def\cX{\mathcal X}

\def\olgamma{\overline{\gamma}}
\def\olGamma{\overline{\Gamma}}
\def\olB{\overline{B}}
\def\olV{\overline{V}}
\def\olN{\overline{N}}
\def\hY{\hat{Y}}
\def\hy{\hat{y}}

\def\mse{\mbox{MSE}}
\def\zo{\mbox{ZO}}

\def\fnr{\mbox{FNR}}
\def\fpr{\mbox{FPR}}

\def\scfnr{\mbox{{\sc fnr}}}
\def\scfpr{\mbox{{\sc fpr}}}
\def\sczo{\mbox{{\sc zo}}}
\def\scmse{\mbox{{\sc mse}}}

\def\ymain{\tilde{y}}

\DeclareMathOperator*{\argmin}{arg\,min}